\pdfoutput = 1
\documentclass[11pt]{article}

\usepackage{fullpage}
\usepackage[utf8]{inputenc} %
\usepackage[T1]{fontenc}    %
\usepackage[pagebackref]{hyperref}       %
\usepackage{url}            %
\usepackage{booktabs}       %
\usepackage{amsfonts}       %
\usepackage{nicefrac}       %
\usepackage{microtype}      %
\usepackage{xcolor}         %

\usepackage{natbib}
\usepackage{amsmath}
\usepackage{amsthm}
\usepackage{thm-restate}
\usepackage{graphicx}
\usepackage{cleveref}

\newcommand{\bR}{\ensuremath{\mathbb{R}}}
\newcommand{\bZ}{\ensuremath{\mathbb{Z}}}
\newcommand{\bE}{\ensuremath{\mathbb{E}}}
\newcommand{\trp}{\ensuremath{\mathsf{T}}}
\newcommand{\diag}{\ensuremath{\mathrm{diag}}}
\newcommand{\ind}{\ensuremath{\mathbf{1}}}
\newcommand{\tr}{\ensuremath{\mathrm {tr}}}
\newcommand{\haar}{\ensuremath{\mathsf{Haar}}}
\newcommand{\kl}{\ensuremath{D_{\mathrm{kl}}}}
\newcommand{\cX}{\ensuremath{\mathcal{X}}}
\newcommand{\cE}{\ensuremath{\mathcal{E}}}
\newcommand{\cV}{\ensuremath{\mathcal{V}}}
\newcommand{\col}{\ensuremath{\mathsf{col}}}
\newcommand{\Ob}{\ensuremath{\mathbb O}}

\newtheorem{theorem}{Theorem}[section]
\newtheorem{assumption}{Assumption}[section]
\newtheorem{definition}{Definition}[section]
\newtheorem{corollary}[theorem]{Corollary}
\newtheorem{claim}[theorem]{Claim}
\newtheorem{lemma}[theorem]{Lemma}

\title{Subspace Recovery from Heterogeneous Data with \\ Non-isotropic Noise}
\date{}
\author{John Duchi \\ Stanford University \and Vitaly Feldman \\ Apple \and Lunjia Hu\thanks{Part of this work was performed while LH was interning at Apple.
LH is also supported by Omer Reingold’s NSF Award IIS-1908774, Omer Reingold’s Simons Foundation Investigators Award 689988, and Moses Charikar's Simons Foundation Investigators Award.} \\ Stanford University \and Kunal Talwar \\ Apple}

\begin{document}
\maketitle
\begin{abstract}
  
Recovering linear subspaces from data is a fundamental and important task in statistics and machine learning. Motivated by heterogeneity in Federated Learning settings, we study a basic formulation of this problem: the principal component analysis (PCA),  with a focus on dealing with irregular noise. Our data come from $n$ users with user $i$ contributing data samples from a $d$-dimensional distribution with mean $\mu_i$. Our goal is to recover the linear subspace shared by $\mu_1,\ldots,\mu_n$ using the data points from all users, where every data point from user $i$ is formed by adding an independent mean-zero noise vector to $\mu_i$. If we only have one data point from every user, subspace recovery is information-theoretically impossible when the covariance matrices of the noise vectors can be non-spherical, necessitating additional restrictive assumptions in previous work. We avoid these assumptions by leveraging at least two data points from each user, which allows us to design an efficiently-computable estimator under non-spherical and user-dependent noise. We prove an upper bound for the estimation error of our estimator in general scenarios where the number of data points and amount of noise can vary across users, and prove an information-theoretic error lower bound that not only matches the upper bound up to a constant factor, but also holds even for spherical Gaussian noise. This implies that our estimator does not introduce additional estimation error (up to a constant factor) due to irregularity in the noise. We show additional results for a linear regression problem in a similar setup.

\end{abstract}
\section{Introduction}
We study the problem of learning low-dimensional structure amongst data distributions, given multiple samples from each distribution. This problem arises naturally in settings such as federated learning, where we want to learn from data coming from a set of individuals, each of which has samples from their own distribution.  These distributions however are related to each other, and in this work, we consider the setting when these distributions have means lying in a low-dimensional subspace. The goal is to learn this subspace, even when the distributions may have different (and potentially non-spherical) variances. This heterogeneity can manifest itself in practice as differing number of samples per user, or the variance differing across individuals, possibly depending on their mean. Recovery of the subspace containing the means can in turn help better estimate individual means. In other words, this can allow for learning good estimator for all individual means, by leveraging information from all the individuals.

The irregularity of the noise makes this task challenging even when we have sufficiently many individual distributions.
For example, suppose we have $n$ individuals and for every $i = 1,\ldots,n$, an unknown $\mu_i\in\bR^d$. For simplicity, suppose that $\mu_1,\ldots,\mu_n$ are distributed independently as $N(0, \sigma^2uu^\trp)$ for $\sigma\in\bR_{\ge 0}$ and an unknown unit vector $u\in\bR^d$. In this setting, our goal is to recover the one-dimensional subspace, equivalently the vector $u$. For every $i$, we have a data point $x_i = \mu_i + z_i$ where $z_i\in\bR^d$ is a mean-zero noise vector. If $z_i$ is drawn independently from a spherical Gaussian $N(0,\alpha^2I)$, we can recover the unknown subspace with arbitrary accuracy as $n$ grows to infinity because  $\frac 1n \sum x_ix_i^\trp$ concentrates to $\bE[x_ix_i^\trp] = \sigma^2uu^\trp + \alpha^2I$, whose top eigenvector is $\pm u$. However, if the noise $z_i$ is drawn from a non-spherical distribution, the top eigenvector of $\frac 1n \sum x_ix_i^\trp$ can deviate from $\pm u$ significantly, and to make things worse, if the noise $z_i$ is drawn independently from a non-spherical Gaussian $N(0, \sigma^2(I - uu^\trp) + \alpha^2I)$, then our data points $x_i = \mu_i + z_i$ distribute independently as $N(0, (\sigma^2 + \alpha^2)I)$, giving no information about the vector $u$.\footnote{This information-theoretic impossibility naturally extends to recovering $k$-dimensional subspaces for $k > 1$ by replacing the unit vector $u\in \bR^d$ with a matrix $U\in \bR^{d\times k}$ with orthonormal columns.}

The information-theoretic impossibility in this example however disappears as soon as one has at least two samples from each distribution. Indeed, given two data points $x_{i1} = \mu_i + z_{i1}$ and $x_{i2} = \mu_i + z_{i2}$ from user $i$, as long as the noise $z_{i1},z_{i2}$ are independent and have zero mean, we always have $\bE[x_{i1}x_{i2}^\trp] = \sigma^2uu^\trp$ regardless of the specific distributions of $z_{i1}$ and $z_{i2}$. This allows us to recover the subspace in this example, as long as we have sufficiently many users each contributing at least two examples. 

As this is commonly the case in our motivating examples, we make this assumption of multiple data points per user, and show that this intuition extends well beyond this particular example. We design efficiently computable estimators for this subspace recovery problem given samples from multiple heteroscedastic distributions (see \Cref{sec:contribution} for details). We prove upper bounds on the error of our estimator measured in the \emph{maximum principal angle} (see \Cref{sec:preliminaries} for definition). We also prove an information-theoretic error lower bound, showing that our estimator achieves the optimal error up to a constant factor in general scenarios where the number of data points and the amount of noise can vary across users. Somewhat surprisingly, our lower bound holds even when the noise distributes as spherical Gaussians. Thus non-spherical noise in setting does not lead to increased error. 

We then show that our techniques extend beyond the mean estimation problem to a linear regression setting where for each $\mu_i$, we get (at least two) samples $(x_{ij},x_{ij}^\trp \mu_i + z_{ij})$ where $z_{ij}$ is zero-mean noise from some noise distribution that depends on $i$ and $x_{ij}$. This turns out to be a model that was recently studied in the meta-learning literature under more restrictive assumptions (e.g.\ $z_{ij}$ is independent of $x_{ij}$) \citep{kong2020meta,tripuraneni2021provable,collins2021exploiting,thekumparampil2021sample}. We show a simple estimator achieving an error upper bound matching the ones in prior work without making these restrictive assumptions.

\subsection{Our contributions}
\label{sec:contribution}
\noindent{\bf PCA with heterogeneous and non-isotropic noise: Upper Bounds.}
In the PCA setting, the data points from each user $i$ are drawn from a user-specific distribution with mean $\mu_i\in\bR^d$, and we assume that $\mu_1,\ldots,\mu_n$ lie in a shared $k$-dimensional subspace that we want to recover. 
Specifically, we have $m_i$ data points $x_{ij}\in\bR^d$ from user $i$ for $j = 1,\ldots,m_i$, and each data point is determined by $x_{ij} = \mu_i + z_{ij}$ where $z_{ij}\in\bR^d$ is a noise vector drawn independently from a mean zero distribution. 
We allow the distribution of $z_{ij}$ to be non-spherical and non-identical across different pairs $(i,j)$.
We use $\eta_i\in\bR_{\ge 0}$ to quantify the amount of noise in user $i$'s data points by assuming that $z_{ij}$ is an \emph{$\eta_i$-sub-Gaussian} random variable.

As mentioned earlier, if we only have a single data point from each user, it is information-theoretically impossible to recover the subspace. Thus, we focus on the case where $m_i\ge 2$ for every $i = 1,\ldots,n$. In this setting, for appropriate weights $w_1,\ldots,w_n\in\bR_{\ge 0}$, we compute a matrix $A$:
\begin{equation}
\label{eq:estimator-PCA}
A = \sum_{i=1}^n \frac{w_i}{m_i(m_i - 1)}\sum_{j_1\ne j_2}x_{ij_1}x_{ij_2}^\trp,
\end{equation}
where the inner summation is over all pairs $j_1,j_2\in\{1,\ldots,m_i\}$ satisfying $j_1\ne j_2$. Our estimator is then defined by the subspace spanned by the top-$k$ eigenvectors of $A$. Although the inner summation is over $m_i(m_i -1)$ terms, the time complexity for computing it need not grow quadratically with $m_i$ because of the following equation:
\[
\sum_{j_1\ne j_2}x_{ij_1}x_{ij_2}^\trp = \left(\sum_{j=1}^{m_i}x_{ij}\right)\left(\sum_{j=1}^{m_i}x_{ij}\right)^\trp - \sum_{j=1}^{m_i}x_{ij}x_{ij}^\trp.
\]

The flexibility in the weights $w_1,\ldots,w_n$ allows us to deal with variations in $m_i$ and $\eta_i$ for different users $i$. In the special case where $\eta_1 = \cdots = \eta_n = \eta$ and $m_1 = \cdots = m_n = m$, we choose $w_1 = \cdots = w_n = 1/n$ and we show that our estimator achieves the following error upper bound with success probability at least $1 - \delta$:
\[
\sin \theta = O\left(\left(\frac{\eta\sigma_1}{\sigma_k^2\sqrt m} + \frac{\eta^2}{\sigma_k^2m}\right)\sqrt{\frac{d + \log(1/\delta)}{n}}\right).
\]
Here, $\theta$ is the maximum principal angle between our estimator and the true subspace shared by $\mu_1,\ldots,\mu_n$, and $\sigma_\ell^2$ is the $\ell$-th largest eigenvalue of $\frac 1n\sum_{i=1}^n \mu_i\mu_i^\trp$. Our error upper bound for general $m_i,\eta_i,w_i$ is given in \Cref{thm:PCA-fixed}.

We instantiate our error upper bound to the case where $\mu_1,\ldots,\mu_n$ are drawn iid from a Gaussian distribution $N(0, \sigma^2 UU^\trp)$, where the columns of $U\in\bR^{d\times k}$ form an orthonormal basis of the subspace containing $\mu_1,\ldots,\mu_n$. By choosing the weights $w_1,\ldots,w_n$ according to $m_1,\ldots,m_n$ and $\eta_1,\ldots,\eta_n$, our estimator achieves the error upper bound
\begin{equation}
\label{eq:intro-PCA-upper-2}
\sin\theta \le O\left(\sqrt{\frac{d + \log(1/\delta)}{\sum_{i=1}^n \gamma_i'}}\right)
\end{equation}
under a mild assumption (Assumption~\ref{assumption}),
where $\gamma_i'$ is defined in \Cref{def:r} and often equals $\left(\frac{\eta_i^2}{\sigma^2m_i} +  \frac{\eta_i^4}{\sigma^4m_i^2}\right)^{-1}$. 

\medskip\noindent{\bf PCA: Lower Bounds.}
We show that the error upper bound \eqref{eq:intro-PCA-upper-2} is optimal up to a constant factor by proving a matching information-theoretic lower bound (\Cref{thm:PCA-lower}). 
Our lower bound holds for general $m_i$ and $\eta_i$ that can vary among users $i$, and it holds even when the noise vectors $z_{ij}$ are drawn from spherical Gaussians, showing that our estimator essentially pays no additional cost in error or sample complexity due to non-isotropic noise.

We prove the lower bound using Fano's method on a local packing over the Grassmannian manifold. 
We carefully select a non-trivial hard distribution so that the strength of our lower bound is not affected by a group of fewer than $k$ users each having a huge amount of data points with little noise.

\medskip\noindent{\bf Linear Models.}
While the PCA setting is the main focus of our paper, 
we extend our research to a related linear models setting that has recently been well studied in the meta-learning and federated learning literature \citep{kong2020meta,tripuraneni2021provable,collins2021exploiting,thekumparampil2021sample}. Here,
the user-specific distribution of each user $i$ is parameterized by $\beta_i\in\bR^d$, and we again assume that $\beta_1,\ldots,\beta_n$ lie in a linear subspace that we want to recover. 
From each user $i$ we observe $m_i$ data points $(x_{ij},y_{ij})\in\bR^d\times \bR$ for $j = 1,\ldots,m_i$ drawn from the user-specific distribution satisfying $y_{ij} = x_{ij}^\trp \beta_i + z_{ij}$ for an $O(1)$-sub-Gaussian measurement vector $x_{ij}\in\bR^d$ with zero mean and identity covariance and an $\eta_i$-sub-Gaussian mean-zero noise term $z_{ij}\in\bR$. 
While it may seem that non-isotropic noise is less of a challenge in this setting since each noise term $z_{ij}$ is a scalar, our goal is to handle a challenging scenario where the variances of the noise terms $z_{ij}$ can depend on the \emph{realized} measurements $x_{ij}$, which is a more general and widely applicable setting compared to those in prior work. 
Similarly to the PCA setting, our relaxed assumptions on the noise make it information-theoretically impossible to do subspace recovery if we only have one data point from each user (see \Cref{sec:linear-model}), and thus we assume each user contributes at least two data points.
We use the subspace spanned by the top-$k$ eigenvectors of the following matrix $A$ as our estimator:
\begin{equation}
\label{eq:estimator-linear}
A = \sum_{i=1}^n \frac{w_i}{m_i(m_i - 1)}\sum_{j_1\ne j_2}(x_{ij_1}y_{ij_1})(x_{ij_2}y_{ij_2})^\trp.
\end{equation}
In the special case where $\eta_1 = \cdots = \eta_n = \eta, m_1 = \cdots = m_n = m$, and $\|\beta_i\|_2 \le r$ for all $i$,
our estimator achieves the following error upper bound:
\begin{equation}
\label{eq:intro-linear-model}
\sin \theta \le O\left(\log^3(nd/\delta)\sqrt{\frac{d(r^4 + r^2\eta^2 + \eta^4/m)}{mn\sigma_k^4}} \right),
\end{equation}
where $\sigma_k^2$ is the $k$-th largest eigenvalue of $\frac 1n\sum_{i=1}^n \beta_i\beta_i^\trp$ (\Cref{cor:linear-model}). Our error upper bound extends smoothly to more general cases where $\eta_i$ and $m_i$ vary among users (\Cref{thm:linear-model}). 
Moreover, our upper bound matches the ones in prior work \citep[e.g.][Theorem 3]{tripuraneni2021provable} despite requiring less restrictive assumptions.

\subsection{Related Work}
Principal component analysis under non-isotropic noise has been studied by  \citet{vaswani2017finite,zhang2018heteroskedastic} and \citet{narayanamurthy2020fast}. When translated to our setting, these papers focus on having only one data point from each user and thus they require additional assumptions---either the level of non-isotropy is low, or the noise is 
coordinate-wise independent and the subspace is incoherent. The estimation error guarantees in these papers depend crucially on how well these additional assumptions are satisfied.
\cite{zhu2019high} and \cite{MR4255114} study PCA with noise and missing data, and \cite{MR4206685} and \cite{MR4345128} study eigenvalue and eigenvector estimation under heteroscedastic noise. These four papers all assume that the noise is coordinate-wise independent and the subspace/eigenspace is incoherent.

The linear models setting we consider has recently been studied as a basic setting of meta-learning and federated learning by \citet{kong2020meta,tripuraneni2021provable,collins2021exploiting}, and \citet{thekumparampil2021sample}. These papers all make the assumption that the noise terms $z_{ij}$ are independent of the measurements $x_{ij}$, an assumption that we relax in this paper. 
\citet{collins2021exploiting} and \citet{thekumparampil2021sample} make improvements in sample complexity and error guarantees compared to earlier work by \citet{kong2020meta} and \citet{tripuraneni2021provable}, but \citet{collins2021exploiting} focus on the noiseless setting ($z_{ij} = 0$) and \citet{thekumparampil2021sample} require at least $\Omega(k^2)$ examples per user. \citet{tripuraneni2021provable} and \citet{thekumparampil2021sample} assume that the measurements $x_{ij}$ are drawn from the standard (multivariate) Gaussian distribution, where as \citet{kong2020meta,collins2021exploiting} and our work make the relaxed assumption that $x_{ij}$ are sub-Gaussian with identity covariance, which, in particular, allows the fourth-order moments of $x_{ij}$ to be non-isotropic.
There is a large body of prior work on meta-learning beyond the linear setting
\citep[see e.g.][]{MR3517104,tripuraneni2020theory,du2020few}.

When collecting data from users, it is often important to ensure that private information about users is not revealed through the release of the learned estimator. Many recent works proposed and analyzed estimators that achieve \emph{user-level differential privacy} in settings including mean estimation \citep{levy2021learning,esfandiari2021tight}, meta-learning \citep{jain2021differentially} and PAC learning \citep{ghazi2021user}. Recently, ~\citet{CummingsFMT21} study one-dimensional mean estimation in a setting similar to ours, under a differential privacy constraint.

The matrix $A$ we define in \eqref{eq:estimator-PCA} is a weighted sum of $A_i:=\frac{1}{m_i(m_i - 1)}\sum_{j_1\ne j_2}x_{ij_1}x_{ij_2}^\trp$ over users $i = 1,\ldots,n$, and each $A_i$ has the form of a $U$-statistic \citep{MR15746,MR26294}. 
$U$-statistics have been applied to many statistical tasks including tensor completion \citep{MR4029842} and various testing problems \citep{MR2816719,MR4206673,schrab2022efficient}.
In our definition of $A_i$, we do \emph{not} make the assumption that the distributions of $x_{i1},\ldots,x_{im_i}$ are identical although the assumption is commonly used in applications of $U$-statistics.
The matrix $A$ in \eqref{eq:estimator-linear} is also a weighted sum of $U$-statistics where we again do not make the assumption of identical distribution.

\subsection{Paper Organization}
In \Cref{sec:preli}, we formally define the maximum principal angle and other notions we use throughout the paper. Our results in the PCA setting and the linear models setting are presented in \Cref{sec:PCA,sec:linear-model}, respectively. We defer most technical proofs to the appendices.
\section{Preliminaries}
\label{sec:preli}
We use $\|A\|$ to denote the spectral norm of a matrix $A$, and use $\|u\|_2$ to denote the $\ell_2$ norm of a vector $u$. For positive integers $k \le d$, we use $\Ob_{d,k}$ to denote the set of matrices $A\in\bR^{d\times k}$ satisfying $A^\trp A = I_k$, where $I_k$ is the $k\times k$ identity matrix. We use $\Ob_d$ to denote $\Ob_{d,d}$, which is the set of $d\times d$ orthogonal matrices.
We use $\col(A)$ to denote the linear subspace spanned by the columns of a matrix $A$.
We use the base-$e$ logarithm throughout the paper.
\label{sec:preliminaries}

\paragraph{Maximum Principal Angle.}
Let $U, \hat U \in \Ob_d$ be two orthogonal matrices. Suppose the columns of $U$ and $\hat U$ are partitioned as $U = [U_1\ U_2], \hat U = [\hat U_1\ \hat U_2]$ where $U_1,\hat U_1\in \Ob_{d,k}$ for an integer $k$ satisfying $0 < k < d$. 
Let $\Gamma$ (resp.\ $\hat \Gamma$) be the $k$-dimensional linear subspace spanned by the columns of $U_1$ (resp.\ $\hat U_1$). 
Originating from \citep{MR1503705},
the \emph{maximum principal angle} $\theta\in [0,\pi/2]$ between $\Gamma$ and $\hat \Gamma$, denoted by $\angle(\Gamma, \hat\Gamma)$ or $\angle(U_1,\hat U_1)$, is defined by $\sin\theta = \|U_1U_1^\trp - \hat U_1\hat U_1^\trp\| = \|U_1^\trp \hat U_2\| = \|U_2^\trp \hat U_1\|$. It is not hard to see that the maximum principal angle depend only on the subspaces $\Gamma,\hat\Gamma$ and not on the choices of $U$ and $\hat U$, and $\sin\angle(\Gamma,\hat \Gamma)$ is a natural metric between $k$-dimensional subspaces (see \Cref{sec:principal-angles} for more details where we discuss the definition of principal angles for any two subspaces with possibly different dimensions).

With the definition of the maximum principal angle, we can now state a variant of the Davis–Kahan $\sin\theta$ theorem \citep{davis1970rotation} that will be useful in our analysis (see \Cref{sec:proof-DK} for proof):
\begin{restatable}[Variant of Davis–Kahan $\sin\theta$ theorem]{theorem}{thmDK}
\label{thm:DK}
Let $A, \hat A\in\bR^{d\times d}$ be symmetric matrices. 
Let $\lambda_i$ denote the $i$-th largest eigenvalue of $A$.
For a positive integer $k$ smaller than $d$,
let $\theta$ denote the maximum principal angle between the subspaces spanned by the top-$k$ eigenvectors of $A$ and $\hat A$.
Assuming $\lambda_k > \lambda_{k+1}$,
\[
\sin\theta \le \frac{2\|A - \hat A\|}{\lambda_k - \lambda_{k+1}} .
\]
\end{restatable}

\paragraph{Sub-Gaussian and sub-exponential distributions.}
We say a random variable $x\in\bR$ with expectation $\bE[x]\in\bR$ has \emph{sub-Gaussian} constant $b\in\bR_{\ge 0}$ if
$\bE[|x - \bE[x]|^p]^{1/p} \le b\sqrt p$ for every $p\ge 1$. We say $x$ has \emph{sub-exponential} constant $b\in\bR_{\ge 0}$ if $\bE[|x - \bE[x]|^p]^{1/p} \le bp$ for every $p \ge 1$. We say a random vector $y\in\bR^d$ has sub-Gaussian (resp.\ sub-exponential) constant $b\in\bR_{\ge 0}$ if for every unit vector $u\in\bR^d$ (i.e., $\|u\|_2 = 1$), the random variable $u^\trp y\in\bR$ has sub-Gaussian (resp.\ sub-exponential) constant $b$. We say $y$ is \emph{$b$-sub-Gaussian} (resp.\ \emph{$b$-sub-exponential}) if it has sub-Gaussian (resp.\ sub-exponential) constant $b$.
\section{Principal Component Analysis}
\label{sec:PCA}
In the principal component analysis (PCA) setting, our goal is to recover the $k$-dimensional subspace $\Gamma$ spanned by the user-specific means $\mu_1,\ldots,\mu_n\in\bR^d$ of the $n$ users. 
From each user $i$, we have $m_i \ge 2$ data points
\begin{equation}
\label{eq:data-pca}
x_{ij} = \mu_i + z_{ij}\quad \textnormal{for}\ j = 1,\ldots,m_i.
\end{equation}
We assume the noise $z_{ij}\in\bR^d$ is drawn independently from a mean zero distribution with sub-Gaussian constant $\eta_i$. We do \emph{not} assume that the variance of $z_{ij}$ is the same along every direction, \emph{nor} do we assume that the distribution of $z_{ij}$ is the same for different $(i,j)$. We first show an error upper bound for our estimator when the user-specific means $\mu_1,\ldots,\mu_n$ are deterministic vectors (\Cref{sec:PCA-fixed}) and then apply this result to the case where $\mu_1,\ldots,\mu_n$ are drawn from a sub-Gaussian distribution (\Cref{sec:PCA-sub-Gaussian}). In \Cref{sec:PCA-lower} we prove an information-theoretic error lower bound matching our upper bound.

\subsection{Fixed User-Specific Means}
\label{sec:PCA-fixed}
We first focus on the case where $\mu_1,\ldots,\mu_n$ are deterministic vectors. In this case, all the randomness in the data comes from the noise $z_{ij}$. Our estimator is the subspace $\hat \Gamma$ spanned by the top-$k$ eigenvectors of $A$ defined in \eqref{eq:estimator-PCA}. For $\ell = 1,\ldots,d$, we define $\sigma_\ell^2$ to be the $\ell$-th largest eigenvalue of $\sum_{i=1}^n w_i\mu_i\mu_i^\trp$. Since $\mu_1,\ldots,\mu_n$ share a $k$-dimensional subspace, $\sigma_{\ell} = 0$ for $\ell > k$. We prove the following general theorem on the error guarantee of our estimator:
\begin{restatable}{theorem}{thmPCAfixed}
\label{thm:PCA-fixed}
Define $\xi^2 = \|\sum_{i=1}^n w_i^2\mu_i\mu_i^\trp \eta_i^2/m_i\|$ and let $\theta$ denote the maximum principal angle between our estimator $\hat \Gamma$ and the true subspace $\Gamma$ spanned by $\mu_1,\ldots,\mu_n$. For any $\delta\in (0,1/2)$, with probability at least $1-\delta$,
\begin{equation}
\label{eq:PCA-fixed}
\sin \theta = O\left(\sigma_k^{-2}\sqrt{(d + \log(1/\delta))\left(\xi^2 + \sum_{i=1}^n \frac{w_i^2\eta_i^4}{m_i^2}\right)} + \sigma_k^{-2}(d + \log(1/\delta))\max_i \frac{w_i\eta_i^2}{m_i}\right).
\end{equation}
\end{restatable}
We can simplify the bound in \Cref{thm:PCA-fixed} by considering special cases:
\begin{corollary}
\label{cor:PCA-fixed}
Assume $\max\{\eta_1/\sqrt{m_1}, \ldots, \eta_n/\sqrt{m_n}\} = t$ and we choose $w_1 = \cdots = w_n = 1/n$. For any $\delta\in (0,1/2)$, with probability at least $1-\delta$,
\begin{equation}
\label{eq:PCA-fixed-1}
\sin \theta = O \left(\frac{t\sigma_1 + t^2}{\sigma_k^2}\sqrt{\frac{d + \log(1/\delta)}{n}}\right).
\end{equation}
In particular, when $\eta_1 = \cdots = \eta_n = \eta$, and $m_1 = \cdots = m_n = m$, error bound \eqref{eq:PCA-fixed-1} becomes
\[
\sin \theta = O\left(\left(\frac{\eta\sigma_1}{\sigma_k^2\sqrt m} + \frac{\eta^2}{\sigma_k^2m}\right)\sqrt{\frac{d + \log(1/\delta)}{n}}\right).
\]
\end{corollary}
We defer the complete proof of \Cref{thm:PCA-fixed} and \Cref{cor:PCA-fixed} to \Cref{sec:proof-PCA-fixed,sec:proof-cor-PCA-fixed}. Our proof is based on the Davis-Kahan $\sin\theta$ theorem (\Cref{thm:DK}). Since $\sigma_{k+1}^2 = 0$, \Cref{thm:DK} implies
\begin{equation}
\label{eq:PCA-DK}
\sin\theta \le \frac{2\|A - \sum_{i=1}^n w_i \mu_i\mu_i^\trp\|}{\sigma_k^2}.
\end{equation}
This reduces our goal to proving an upper bound on the spectral norm of $A - \sum_{i=1}^n w_i\mu_i\mu_i^\trp$. 
Since for distinct $j_1$ and $j_2$ in $\{1,\ldots,m_i\}$ we have $\bE[x_{ij_1}x_{ij_2}^\trp] = \mu_i\mu_i^\trp$, our construction of $A$ in \eqref{eq:estimator-PCA} guarantees $\bE[A] = \sum_{i=1}^nw_i\mu_i\mu_i^\trp$. Therefore, our goal becomes controlling the deviation of $A$ from its expectation, and we achieve this goal using techniques for matrix concentration inequalities.

\subsection{Sub-Gaussian User-Specific Means}
\label{sec:PCA-sub-Gaussian}
We apply our error upper bound in \Cref{thm:PCA-fixed} to the case where $\mu_1,\ldots,\mu_n\in\bR^d$ are drawn iid from $N(0,\sigma^2 UU^\trp)$ for an unknown $U\in \Ob_{d,k}$. 
We still assume that each data point $x_{ij}\in\bR^d$ is generated by adding a noise vector $z_{ij}\in\bR^d$ to the user-specific mean $\mu_i$ as in \eqref{eq:data-pca}. We do \emph{not} assume that the noise vectors $(z_{ij})_{1\le i \le n, 1 \le j \le m_i}$ are independent of the user-specific means $(\mu_i)_{1\le i \le n}$, but we assume that when conditioned on $(\mu_i)_{1 \le i \le n}$, every noise vector $z_{ij}$ independently follows a distribution with mean zero and sub-Gaussian constant $\eta_i$.
We use the same estimator $\hat \Gamma$ as before: $\hat\Gamma$ is the subspace spanned by the top-$k$ eigenvectors of $A$ defined in \eqref{eq:estimator-PCA}. We determine the optimal weights $w_1,\ldots,w_n$ in \eqref{eq:estimator-PCA} as long as $m_1,\ldots,m_n$ and $\eta_1,\ldots,\eta_n$ satisfy a mild assumption (Assumption~\ref{assumption}), achieving an error upper bound in \Cref{thm:PCA-weights-chosen}. In the next subsection, we prove an error lower bound (\Cref{thm:PCA-lower}) that matches our upper bound (\Cref{thm:PCA-weights-chosen}) up to a constant factor, assuming $d \ge (1 + \Omega(1))k$ and $\delta = \Theta(1)$.

We prove our error upper bound in a slightly more general setting than $\mu_1,\ldots,\mu_n$ drawn iid from $N(0,\sigma^2UU^\trp)$. Specifically, we make the following assumption on the distribution of $\mu_1,\ldots,\mu_n$:
\begin{assumption}
\label{assumption:mu}
The user-specific means $\mu_1,\ldots,\mu_n\in\bR^d$ are mean-zero independent random vectors supported on an unknown $k$-dimensional subspace $\Gamma$. Moreover, for a parameter $\sigma > 0$, for every $i = 1,\ldots,n$, $\mu_i$ has sub-Gaussian constant $O(\sigma)$, and the $k$-th largest eigenvalue of $\bE[\mu_i\mu_i^\trp]$ is at least $\sigma^2$. 
\end{assumption}
Under this assumption, we have the following lower bound on the $\sigma_k^2$ in \Cref{thm:PCA-fixed} (see \Cref{sec:proof-sigma-k} for proof):

\begin{claim}
\label{claim:sigma-k}
Under Assumption~\ref{assumption:mu}, let $w_1,\ldots,w_n\in\bR_{\ge 0}$ be user weights satisfying $w_1 + \cdots + w_n =1$ and $\sigma_k^2$ be the $k$-th largest eigenvalue of $\sum_{i=1}^n w_i\mu_i\mu_i^\trp$.
There exists an absolute constant $C_* > 1$ such that for any $\delta\in (0,1/2)$, as long as $\max_{1\le i \le n}w_i \le 1/C_*(k + \log(1/\delta))$, then $\sigma_k^2 \ge \sigma^2/2$ with probability at least $1 - \delta/2$.
\end{claim}

The following definition is important for us to choose the weights $w_1,\ldots,w_n$ in \eqref{eq:estimator-PCA} optimally:
\begin{definition}
\label{def:r}
Define $\gamma_i = \left(\frac{\eta_i^2}{\sigma^2m_i} + \frac{\eta_i^4}{\sigma^4m_i^2}\right)^{-1}$ and assume w.l.o.g.\ that $\gamma_1\ge \cdots \ge \gamma_n$.
Define $\gamma_i' = \gamma_i$ if $i\ge k$, and $\gamma_i' = \gamma_k$ if $i < k$.
\end{definition}
Intuitively, we can view $\gamma_i$ as measuring the ``amount of information'' provided by the data points from user $i$. This is consistent with the fact that $\gamma_i$ increases as the number $m_i$ of data points from user $i$ increases, and $\gamma_i$ decreases as the noise magnitude $\eta_i$ from user $i$ increases.
With the users sorted so that $\gamma_1\ge \cdots \ge \gamma_n$, the quantity $\gamma_i'$ is then defined to be $\gamma_k$ for the $k$ most ``informative'' users $i = 1,\ldots,k$, and $\gamma_i' = \gamma_i$ for other users.
We make the following mild assumption on $\gamma_i'$ under which we achieve optimal estimation error:
\begin{assumption}
\label{assumption}
$\sum_{i=1}^n \gamma_i' \ge C_*(k + \log(1/\delta))\gamma_1'$ for $C_*$ defined in Claim~\ref{claim:sigma-k}.
\end{assumption}
By the definition of $\gamma_i'$, it is easy to show that Assumption~\ref{assumption} is equivalent to $\sum_{i=k + 1}^n \gamma_i \ge ((C_* - 1) k + C_*\log(1/\delta))\gamma_k$. 
Therefore, if we view $\gamma_i$ as the ``amount of information'' from user $i$, Assumption~\ref{assumption} intuitively requires that a significant contribution to the total ``information'' comes from outside the $k$ most ``informative'' users. This assumption allows us to avoid the case where we only have exactly $n = k$ users: in that case, we would have $\sigma_k^2\approx \sigma^2/k^2$ for uniform weights $w_1 = \cdots = w_n$ (see \citep{MR2407948} and references therein), as opposed to the desired $\sigma_k^2 \ge \sigma^2/2$ in \Cref{claim:sigma-k}.

Assumption~\ref{assumption} is a mild assumption. For example,
when $\gamma_{k} = \cdots = \gamma_n$, Assumption~\ref{assumption} holds as long as $n \ge C_*(k + \log(1/\delta))$.
Also,
since $\gamma_1' = \cdots = \gamma_k' \ge \gamma_{k+1}' \ge \cdots \ge \gamma_n' \ge 0$, it trivially holds that $\sum_{i=1}^n \gamma_i' \ge k\gamma_1'$. Assumption~\ref{assumption} is relatively mild when compared to this trivial inequality.

Under Assumption~\ref{assumption}, we show that it is optimal to choose the weights $w_1,\ldots,w_n$ as 
\begin{equation}
\label{eq:optimal-weights}
w_i = \frac{\gamma_i'}{\sum_{\ell = 1}^n \gamma_\ell'}. 
\end{equation}
Specifically, if we plug \eqref{eq:optimal-weights} into \Cref{thm:PCA-fixed} and bound $\xi$ and $\sigma_k$ based on the distribution of $\mu_1,\ldots,\mu_n$,
we get the following error upper bound which matches our lower bound (\Cref{thm:PCA-lower}) in \Cref{sec:PCA-lower}. We defer its proof to \Cref{sec:proof-PCA-weights-chosen}.
\begin{restatable}{theorem}{thmPCAweightschosen}
\label{thm:PCA-weights-chosen}
Under Assumptions~\ref{assumption:mu} and \ref{assumption}, if we choose $w_1,\ldots,w_n$ as in \eqref{eq:optimal-weights} and define $\theta = \angle(\Gamma,\hat\Gamma)$, for $\delta\in (0,1/2)$, with probability at least $1-\delta$,
\begin{equation}
\label{eq:PCA-weights-chosen}
\sin\theta \le O\left(\sqrt{\frac{d + \log(1/\delta)}{\sum_{i=1}^n \gamma_i'}}\right).
\end{equation}
\end{restatable}
For comparison, consider the setting when $\sigma=\eta_i=1$ for every $i = 1,\ldots,n$. The result then says that $\sin\theta$ is bounded by approximately $\sqrt{\frac{d}{\sum_{i=1}^n m_i}}$. This is the same rate as we would get if we have $\sum_{i=1}^n m_i$ users each contributing a single independent data point with homogeneous spherical noise. Thus as long as the data points are not too concentrated on fewer than $k$ users, the heterogeneity comes at no additional cost.

\subsection{Lower Bound}
\label{sec:PCA-lower}
We prove a lower bound matching the upper bound in \Cref{thm:PCA-weights-chosen} up to constant in the setting where $\delta = \Theta(1)$, $d \ge (1 + \Omega(1))k$. 

For every positive integer $d$, there is a natural ``uniform'' distribution over $\Ob_d$ given by Haar's theorem \citep{MR1503103} (see e.g.\ \citep{MR3186070} for a textbook). We denote this distribution by $\haar(\Ob_d)$. A random matrix $A$ drawn from $\haar(\Ob_d)$ has the following invariance property: for any deterministic matrix $B\in \Ob_d$, the random matrices $A, AB$ and $BA$ all have the same distribution. For an integer $k \le d$, we can construct a random matrix $A_1\in \Ob_{d,k}$ by first drawing $A\in\bR^{d\times d}$ from $\haar(\Ob_d)$ and then take the first $k$ columns of $A$. We denote the distribution of $A_1$ by $\haar(\Ob_{d,k})$. The invariance property of $\haar(\Ob_d)$ immediately implies the following claims:
\begin{claim}
\label{claim:invariance-1}
Let $A\in \Ob_d$ be a random matrix drawn from $\haar(\Ob_d)$ and let $B\in \Ob_{d,k}$ be a fixed matrix. Then $AB$ distributes as $\haar(\Ob_{d,k})$.
\end{claim}
\begin{proof}
The matrix $B$ can be written as the first $k$ columns of a matrix $C\in \Ob_{d}$. Now $AB$ is the first $k$ columns of $AC$, where $AC$ distributes as $\haar(\Ob_d)$ by the invariance property. This implies that $AB$ distributes as $\haar(\Ob_{d,k})$.
\end{proof}
\begin{claim}
\label{claim:invariance-2}
Let $B\in \Ob_{d,k}$ be a random matrix. Assume for every fixed matrix $A\in \Ob_d$, the random matrices $B$ and $AB$ have the same distribution. Then $B\sim \haar(\Ob_{d,k})$.
\end{claim}
\begin{proof}
If we draw $A$ independently from $\haar(\Ob_d)$, the random matrices $B$ and $AB$ still have the same distribution. By Claim~\ref{claim:invariance-1}, $AB$ distributes as $\haar(\Ob_{d,k})$, so $B$ must also distribute as $\haar(\Ob_{d,k})$.
\end{proof}
With the definition of $\haar(\Ob_{d,k})$, we state our lower bound in the following theorem:
\begin{restatable}{theorem}{pcalb}
\label{thm:PCA-lower}
Let $k,d,n$ be positive integers satisfying $k < d$ and $k \le n$. Let $m_1,\ldots,m_n$ be positive integers and $\sigma,\eta_1,\ldots,\eta_n$ be positive real numbers.
Suppose we draw $U\in \Ob_{d,k}$ from $\haar(\Ob_{d,k})$ and then draw $\mu_1,\ldots,\mu_n$ independently from
$N(0, \sigma^2UU^\trp)$. 
For every $i = 1,\ldots,n$, we draw $m_i$ data points $x_{ij}$ for $j = 1,\ldots,m_i$ as $x_{ij} = \mu_i + z_{ij}$, where
each $z_{ij}$ is drawn independently from the spherical Gaussian $N(0, \eta_i^2I)$.
Let $\hat\Gamma$ be any estimator mapping $(x_{ij})_{1\le i \le n, 1\le j \le m_i}$ to a (possibly randomized) $k$-dimensional subspace of $\bR^d$.
Let $\theta$ denote the maximum principal angle between $\hat \Gamma((x_{ij})_{1\le i \le n, 1\le j \le m_i})$ and the true subspace $\Gamma = \col(U)$.
If real numbers $t \ge 0$ and $\delta \in [0,1/2)$ satisfy
$
\Pr[\sin\theta \le t] \ge 1-\delta,
$
then
\begin{equation}
\label{eq:PCA-lower}
t \ge \Omega\left(\min\left\{1,\sqrt{\frac{(d - k)(1 - \delta)}{\sum_{i=k}^{n} \gamma_i}}\right\}\right),
\end{equation}
where $\gamma_1,\ldots,\gamma_n$ are defined in \Cref{def:r}.
\end{restatable}
Note that $\gamma_i' = \gamma_i$ for $i \ge k$, so our upper bound in \eqref{eq:PCA-weights-chosen} matches the lower bound \eqref{eq:PCA-lower} up to a constant factor assuming $\delta = \Theta(1)$ and $d \ge (1 + \Omega(1))k$.

We use the local Fano method to prove the lower bound using the technical lemmas in \Cref{sec:fano}. In particular, we reduce our goal to proving an upper bound on the KL divergence between Gaussian distributions whose covariance matrices are defined based on matrices $U,\hat U\in \Ob_{d,k}$ with $\|UU^\trp - \hat U\hat U^\trp\|_F$ bounded. We prove the following lemma in \Cref{sec:proof-kl-angle} that upper bounds the KL divergence using $\|UU^\trp - \hat U\hat U^\trp\|_F$:
\begin{lemma}
\label{lm:kl-angle}
For $\sigma\in\bR_{\ge 0}, \eta\in\bR_{>0},U,\hat U\in \Ob_{d,k}$, define $\Sigma = \sigma^2UU^\trp + \eta^2 I$ and $\hat \Sigma = \sigma^2\hat U\hat U^\trp + \eta^2 I$. Then,
\[
\kl(N(0, \hat \Sigma)\|N(0, \Sigma)) = \frac {\sigma^4\|UU^\trp - \hat U\hat U^\trp\|_F^2}{4(\sigma^2\eta^2 + \eta^4)}.
\]
\end{lemma}
\Cref{lm:kl-angle} and the results in \Cref{sec:fano} allow us to prove a version of \eqref{eq:PCA-lower} in which the sum in the demoninator is over $i = 1,\ldots,n$.
This, however, is weaker and less useful than \eqref{eq:PCA-lower} in which the sum in the denominator is over $i = k,k+1,\ldots,n$.
To prove \Cref{thm:PCA-lower}, we extract a hard distribution in which the data points from users $1,\ldots,k-1$ are ``useless'' in terms of subspace recovery. 

Let $\Gamma_1$ be the $(k-1)$-dimensional subspace spanned by $\mu_1,\ldots,\mu_{k-1}$. We let $v_1,\ldots,v_{k-1}$ be a random orthonormal basis of $\Gamma_1$, and we append another vector $v_k\in\Gamma$ to form an orthonormal basis $v_1,\ldots,v_k$ of $\Gamma$. We define $V_1 = [v_1\ \cdots\ v_{k-1}]\in \Ob_{d,k-1}$ and $V = [v_1\ \cdots\ v_k]\in \Ob_{d,k}$. In \Cref{fig:1} we show a graphical model demonstrating the dependency among the random objects we defined.
\begin{figure}[h]
\centering
\includegraphics[scale=0.3]{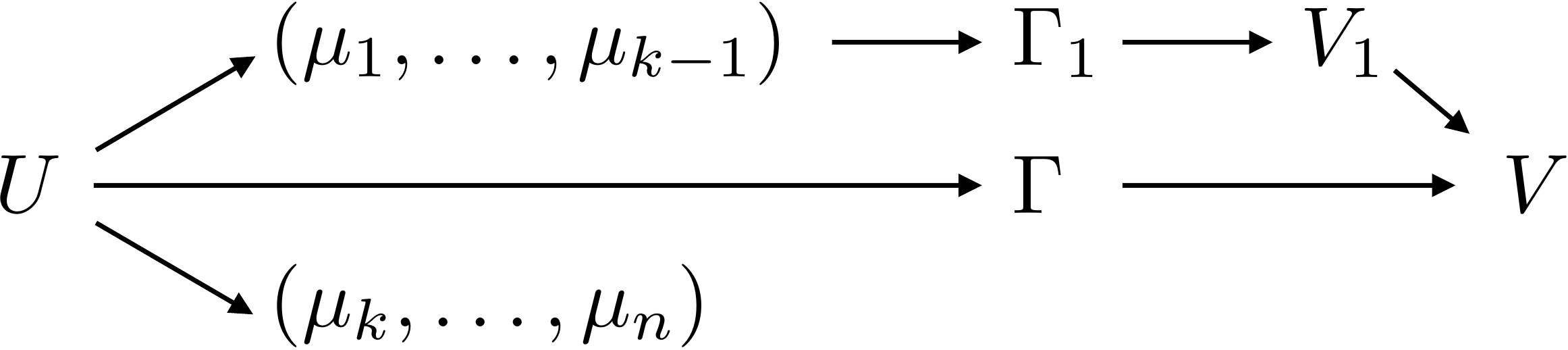}
\caption{Graphical model A.}
\label{fig:1}
\end{figure}

Let us focus on the joint distribution of $(V_1,V, (\mu_1,\ldots,\mu_{k-1}))$. 
By the invariance property, for any matrices $\tilde V_1\in \Ob_{d,k-1}, \tilde V\in \Ob_{d,k}$, measurable set $S\subseteq (\bR^d)^{k - 1}$, and orthogonal matrix $G\in \Ob_d$,
\[
\Pr[(\mu_1,\ldots,\mu_{k-1})\in S| V = \tilde V, V_1 = \tilde V_1] = \Pr[(\mu_1,\ldots,\mu_{k-1})\in S_G| V = G\tilde V, V_1 = G\tilde V_1],
\]
where $S_G = \{(G\tilde\mu_1,\ldots,G\tilde\mu_{k-1}):(\tilde\mu_1,\ldots,\tilde\mu_{k-1})\in S\}$.
For any $\tilde V,\tilde V'\in \Ob_{d,k}$ whose first $k-1$ columns are both $\tilde V_1$, there exists $G\in \Ob_d$ such that $\tilde V' = G\tilde V$  and thus $\tilde V_1 = G\tilde V_1$. This implies that for any $\tilde\mu\in \col(\tilde V_1)$, we have $G\tilde\mu = \tilde\mu$, and thus $(S\cap \col(\tilde V_1)^{k-1})_G = S\cap \col(\tilde V_1)^{k-1}$ for any measurable $S\subseteq (\bR^d)^{k-1}$. 
Here, $\col(\tilde V_1)^{k-1} = \{(\tilde \mu_1,\ldots,\tilde \mu_{k-1}): \tilde \mu_i\in \col(\tilde V_1) \text{ for } i=1,\ldots,k-1\}\subseteq (\bR^d)^{k-1}$.
When conditioned on $V_1 = \tilde V_1$, for every $i = 1,\ldots,k-1$ we have $\mu_i \in \Gamma_1 = \col(V_1) = \col(\tilde V_1)$, which implies that $(\mu_1,\ldots,\mu_{k-1})\in \col(\tilde V_1)^{k-1}$.
Therefore,
\begin{align*}
&\Pr[(\mu_1,\ldots,\mu_{k-1})\in S| V = \tilde V, V_1 = \tilde V_1]\\
= {} & \Pr[(\mu_1,\ldots,\mu_{k-1})\in S\cap \col(\tilde V_1)^{k-1}| V = \tilde V, V_1 = \tilde V_1]\\
= {} & \Pr[(\mu_1,\ldots,\mu_{k-1})\in (S\cap \col(\tilde V_1)^{k-1})_G| V = G\tilde V, V_1 = G\tilde V_1]\\
= {} & \Pr[(\mu_1,\ldots,\mu_{k-1})\in S\cap \col(\tilde V_1)^{k-1}| V = \tilde V', V_1 = \tilde V_1]\\
= {} & \Pr[(\mu_1,\ldots,\mu_{k-1})\in S| V = \tilde V', V_1 = \tilde V_1].
\end{align*}
This implies that $(\mu_1,\ldots,\mu_{k-1})$ and $V$ are conditionally independent given $V_1$. Therefore, the joint distribution of $(V_1,V, (\mu_1,\ldots,\mu_{k-1}))$ can be formed by first drawing $V$ and $V_1$, and then drawing $\mu_1,\ldots,\mu_{k-1}$ based only on $V_1$ and not on $V$. Since $\mu_k,\ldots,\mu_n$ are drawn iid from $N(0, \sigma^2UU^\trp) = N(0, \sigma^2VV^\trp)$, we have the graphical model shown in \Cref{fig:2}.
\begin{figure}[h]
\centering
\includegraphics[scale=0.3]{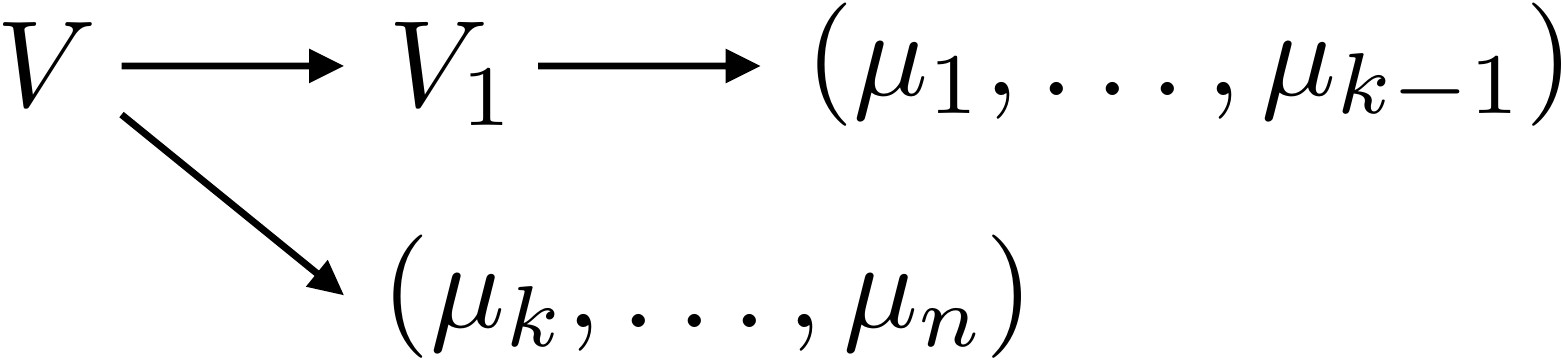}
\caption{Graphical model B.}
\label{fig:2}
\end{figure}

By Claim~\ref{claim:invariance-2}, the marginal distribution of $V$ is $\haar(\Ob_{d,k})$.
By Claim~\ref{claim:invariance-1}, we can implement this distribution by first drawing $W\sim \haar(\Ob_d)$ and then drawing $E$ independently from any distribution over $\Ob_{d,k}$ and let $V = WE$. We choose the distribution of $E$ later, where we ensure that the first $k-1$ columms of $E$ is always $\begin{bmatrix}I_{k-1} \\ 0\end{bmatrix}$. This guarantees that the first $k-1$ columns of $W$ and $V$ are the same, and thus $V_1$ is exactly the first $k-1$ columns of $W$, resulting in the graphical model shown in \Cref{fig:3}.
\begin{figure}[h]
\centering
\includegraphics[scale=0.3]{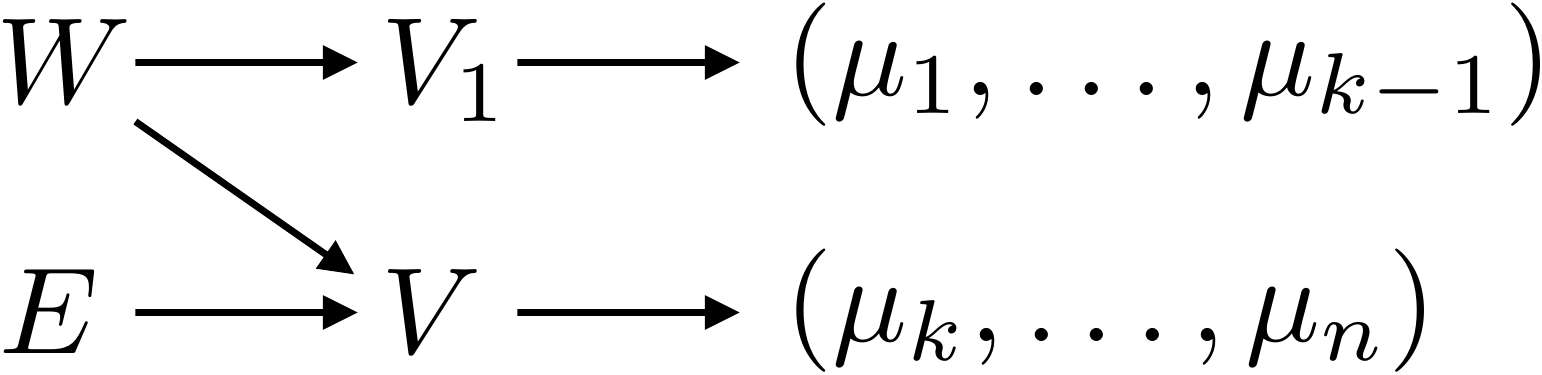}
\caption{Graphical model C.}
\label{fig:3}
\end{figure}

Note that in \Cref{fig:3} there is no directed path from $E$ to $(\mu_1,\ldots,\mu_{k-1})$. Intuitively, this means that knowing $(\mu_1,\ldots,\mu_{k-1})$ gives us no information about $E$. 
Now by choosing the distribution of $E$ appropriately, we can prove \eqref{eq:PCA-lower} in which the denominator does not contain $\gamma_1,\ldots,\gamma_{k-1}$.
We defer the complete proof of \Cref{thm:PCA-lower} to \Cref{sec:proof-PCA-lower}.

\section{Linear Models}
\label{sec:linear-model}
In the linear models setting, 
the data distribution of user $i$ is parameterized by an unknown vector $\beta_i\in\bR^d$. As before, we assume that the vectors $\beta_1,\ldots,\beta_n$ from the $n$ users lie in an unknown $k$-dimensional subspace $\Gamma$.
Our goal is to recover the subspace using the following data.
For every $i = 1,\ldots,n$, we have $m_i$ data points from user $i$: $(x_{i1},y_{i1}),\ldots,(x_{im_i}, y_{im_i})\in\bR^d\times\bR$. For every $j = 1,\ldots,m_i$, we assume the \emph{measurement} $x_{ij}\in\bR^d$ is a random vector drawn independently from an $O(1)$-sub-Gaussian distribution with zero mean and identity covariance matrix. The measurement \emph{outcome} $y_{ij}$ is determined by $y_{ij}= x_{ij}^\trp \beta_i + z_{ij}$, where
the random \emph{noise} $z_{ij}\in\bR$ can depend on the measurements $x_{i1},\ldots,x_{im_i}$. When conditioned on $x_{i1},\ldots,x_{im_i}$, we assume
every $z_{ij}$ for $j = 1,\ldots, m_i$ is independently drawn from an $\eta_{i}$-sub-Gaussian distribution with zero mean, but we do \emph{not} assume that the conditional distribution of $z_{ij}$ is the same for every $j = 1,\ldots,m_i$. The (in)dependence among $x_{ij}$ and $z_{ij}$ for $i = 1,\ldots,n$ and $j =1,\ldots, m_i$ can be summarized by the example graphical model in \Cref{fig:4}.
\begin{figure}[h]
\centering
\includegraphics[scale=0.3]{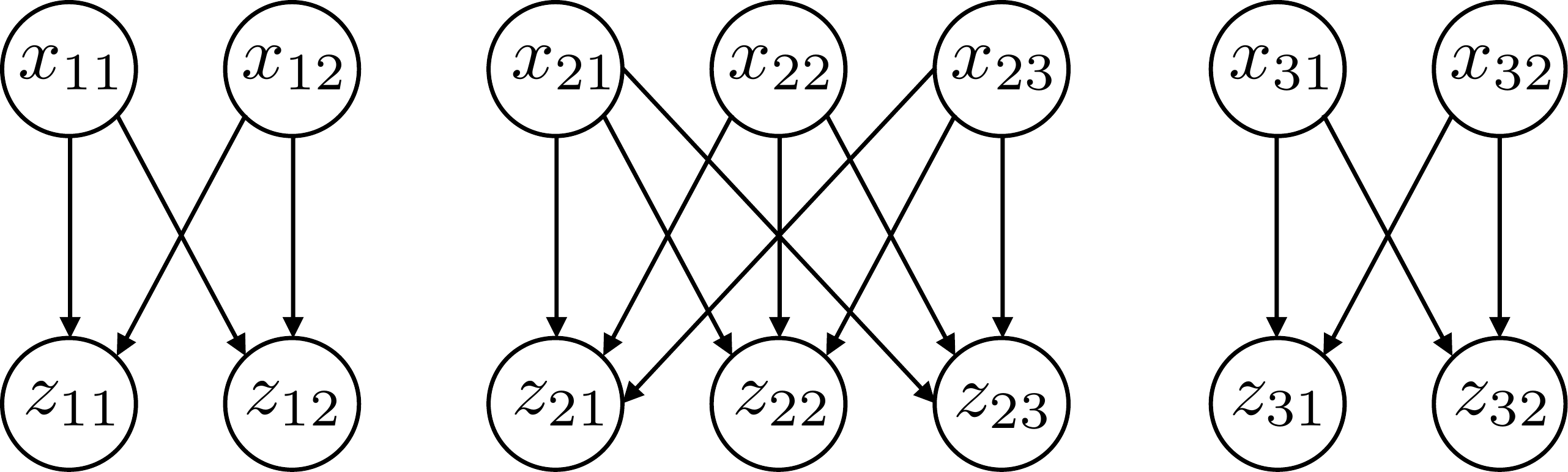}
\caption{An example for $n = 3, m_1 = 2, m_2 = 3, m_3 = 2$.}
\label{fig:4}
\end{figure}

Since we allow the noise $z_{ij}$ to depend on the measurements $x_{ij}$, it is information-theoretically impossible to recover the subspace if we only have one data point from every user. 
Consider the scenario where every $\beta_i$ is drawn independently from $N(0, \sigma^2uu^\trp)$ for an unknown unit vector $u\in\bR^d$ and every $x_{ij}$ is drawn independently and uniformly from $\{-1,1\}^d$. If we set $z_{ij}$ to be $z_{ij} = x_{ij}^\trp \nu_{ij}$ where $\nu_{ij}$ is independently drawn from $N(0, \sigma^2(I - uu^\trp))$, then every $y_{ij}$ satisfies $y_{ij} = x_{ij}^\trp (\beta_i + \nu_{ij})$ where $\beta_i + \nu_{ij}$ distributes as $N(0,\sigma^2I)$ independently from $x_{ij}$. This implies that the joint distribution of $((x_{i1},y_{i1}))_{i=1,\ldots,n}$ does not change with $u$, i.e., we get no information about $u$ from one data point per user.

Thus, we assume $m_i \ge 2$ for every user $i$. In this case, we achieve error upper bounds that match the ones in \citep{tripuraneni2021provable}
despite our relaxed assumptions on the noise. 
Our estimator is the subspace $\hat \Gamma$ spanned by the top-$k$ eigenvectors of $A$ defined in \eqref{eq:estimator-linear}.
We defer the analysis of our estimator to \Cref{sec:linear-models-details}.

\bibliographystyle{plainnat}
\appendix
\section{Basics about Principal Angles}
\label{sec:principal-angles}
We provide a formal definition of the principal angles introduced by \citet{MR1503705} and briefly discuss some basic properties of this notion (see e.g.\ \citep{MR1061154} for a textbook).

Let $U, \hat U \in \Ob_d$ be two orthogonal matrices. Suppose the columns of $U$ and $\hat U$ are partitioned as $U = [U_1\ U_2], \hat U = [\hat U_1\ \hat U_2]$ where $U_1\in \Ob_{d,k}$ and $\hat U_1\in \Ob_{d,\hat k}$ for integers $k,\hat k$ satisfying $0 < k \le \hat k < d$. Starting from a singular value decomposition of $U_1^\trp \hat U_1$ as $R_1^\trp\Sigma \hat R_1$, one can find an integer $\ell\in\bZ_{\ge 0}$ and angles $0 < \theta_1\le \cdots \le \theta_\ell \le \pi/2$ such that there exist orthogonal matrices $P,\hat P\in \Ob_d, R_1\in \Ob_k, R_2\in \Ob_{d - k}, \hat R_1\in \Ob_{\hat k}, \hat R_2\in \Ob_{d - \hat k}$ satisfying 
\begin{align}
& U = P\begin{bmatrix} R_1 & 0 \\ 0 & R_2\end{bmatrix},\quad \
\hat U = \hat P\begin{bmatrix} \hat R_1 & 0 \\ 0 & \hat R_2\end{bmatrix}, \quad\textnormal{and} \label{eq:principal-1}
\\
& P^\trp \hat P = (\hat P^\trp P)^\trp = \begin{bmatrix}
I_{k - \ell} & 0 & 0 & 0 & 0 \\
0 & \cos\Theta & 0 & \sin\Theta & 0\\
0 & 0 & I_{\hat k - k} & 0 & 0\\
0 & -\sin\Theta & 0 & \cos\Theta & 0\\
0 & 0 & 0 & 0 & I_{d - \hat k - \ell}
\end{bmatrix},\label{eq:principal-2}
\end{align}
where $\cos\Theta = \diag(\cos\theta_1,\ldots,\cos\theta_\ell)$ and $\sin \Theta = \diag(\sin\theta_1,\ldots,\sin\theta_\ell)$. It is easy to see that $\cos\theta_1,\ldots,\cos\theta_\ell$ are exactly the singular values of $U_1^\trp \hat U_1$ that are not equal to $1$, so $\ell, \theta_1,\ldots,\theta_\ell$ are unique.
Let $\Gamma$ (resp.\ $\hat \Gamma$) be the $k$-dimensional (resp.\ $\hat k$-dimensional) linear subspace spanned by the columns of $U_1$ (resp.\ $\hat U_1$).
The (non-zero) \emph{principal angles} between $\Gamma$ and $\hat \Gamma$ are defined to be $\theta_1,\ldots,\theta_{\ell}$. It is not hard to see that the principal angles depend only on the subspaces $\Gamma,\hat\Gamma$ and not on the choices of $U$ and $\hat U$. 
The \emph{maximum principal angle} between $\Gamma$ and $\hat \Gamma$, denoted by $\angle(\Gamma, \hat\Gamma)$ or $\angle(U_1,\hat U_1)$, is defined to be $\theta_\ell$ (or $0$ when $\ell = 0$). When $k = \hat k$, using \eqref{eq:principal-1} and \eqref{eq:principal-2}, it is not hard to show that 
the non-zero eigenvalues of $U_1U_1^\trp - \hat U_1 \hat U_1^\trp$ are $\pm \sin\theta_1,\ldots,\pm\sin\theta_\ell$.
This implies
$\|U_1U_1^\trp - \hat U_1 \hat U_1^\trp\| = \sin\angle(\Gamma,\hat \Gamma)$ and $\|U_1U_1^\trp - \hat U_1 \hat U_1^\trp\|_F = \sqrt{2\sum_{i=1}^\ell \sin^2\theta_i}$. In particular, $\sin\angle(\Gamma,\hat \Gamma)$ is a natural metric between $k$-dimensional subspaces.
\section{Basic Facts and Concentration Inequalities}
\begin{lemma}[{\citep[See][Proposition 2.6.1]{vershynin2018high}}]
\label{lm:sub-Gaussian-sum}
If $X_1,\ldots,X_n\in\bR$ are independent random variables with sub-Gaussian constants $b_1,\ldots,b_n\in\bR_{\ge 0}$ respectively, then $\sum_{i=1}^n X_i$ has sub-Gaussian constant $O\left(\sqrt{\sum_{i=1}^n b_i^2}\right)$.
\end{lemma}
\begin{lemma}[{\citep[See][Lemma 2.7.6]{vershynin2018high}}]
\label{lm:sub-Gaussian-square}
If $X\in\bR$ is a mean-zero random variable with sub-Gaussian constant $b\in\bR_{\ge 0}$, then $X^2$ has sub-exponential constant $O(b^2)$.
\end{lemma}
\begin{lemma}[{\citep[See][Theorem 2.6.2]{vershynin2018high}}]
\label{lm:sub-gaussian}
Let $x_1,\ldots,x_n\in\bR$ be independent random variables. For every $i = 1,\ldots,n$, assume $x_i$ has mean zero and sub-Gaussian constant $b_i\in\bR_{\ge 0}$. Then for every $\delta\in (0,1/2)$, with probability at least $1-\delta$,
\[
\left|\sum_{i=1}^n x_i\right| \le O\left(\sqrt{\log(1/\delta)\sum_{i=1}^n b_i^2}\right).
\]
\end{lemma}

\begin{lemma}[{\citep[See][Theorem 2.8.1]{vershynin2018high}}]
\label{lm:sub-exponential}
Let $x_1,\ldots,x_n\in\bR$ be independent random variables. For every $i = 1,\ldots,n$, assume $x_i$ has mean zero and sub-exponential constant $b_i\in\bR_{\ge 0}$. Then for every $\delta\in (0,1/2)$, with probability at least $1-\delta$,
\[
\left|\sum_{i=1}^n x_i\right| \le O\left(\sqrt{\log(1/\delta)\sum_{i=1}^n b_i^2} + \log(1/\delta)\max_{i=1,\ldots,n} b_i\right).
\]
\end{lemma}

\section{Vector and Matrix Concentration Inequalities}

\begin{lemma}[{\citep[See][Corollary 4.2.13 and Exercise 4.4.3]{vershynin2018high}}]
\label{lm:covering}
There exists an absolute constant $C > 1$ and a set $O'_d\subseteq \Ob_{d,1}$ for every positive integer $d$ such that
\begin{enumerate}
\item $|O'_d| \le 2^{Cd}$ for every $d\in \bZ_{>0}$.
\item for every $d\in \bZ_{>0}$ and every symmetric matrix $A\in\bR^{d\times d}$,
\[
\|A\|\le C \sup_{u\in O'} |u^\trp A u|.
\]
\item for every $d_1,d_2\in \bZ_{>0}$ and every matrix $A\in \bR^{d_1\times d_2}$, 
\[
\|A\| \le C \sup_{u_1\in O'_{d_1}}\sup_{u_2\in O'_{d_2}}|u_1^\trp A u_2|.
\]
\end{enumerate}
\end{lemma}

\begin{lemma}
\label{lm:matrix-concentration}
Suppose $x_1,\ldots,x_n\in\bR^d$ are independent random vectors, and each $x_i$ is mean-zero and $1$-sub-Gaussian. Suppose $w_1,\ldots,w_n\in\bR$ are fixed real numbers. Define $E = \sum_{i=1}^n w_ix_ix_i^\trp$ and $\bar E = \bE[E]$. Then for any $\delta\in (0,1/2)$, with probability at least $1-\delta$,
\[
\|E - \bar E\| \le O\left(\sqrt{(d + \log(1/\delta))\sum_{i=1}^n w_i^2} + (d + \log(1/\delta))\max_i|w_i|\right).
\]
\end{lemma}
\begin{proof}
Let $O'$ denote the set $O'_d\subseteq \Ob_{d,1}$ guaranteed by \Cref{lm:covering}. For any fixed $u\in O'$, by \Cref{lm:sub-Gaussian-square}, $(x_i^\trp u)^2$ has sub-exponential constant $O(1)$, and thus
$w_i(x_i^\trp u)^2$ has sub-exponential constant $|w_i|$. 
Define $\delta':=\delta/|O'|$. By \Cref{lm:sub-exponential}, with probability at least $1-\delta'$,
\begin{equation}
\label{eq:proof-matrix-concentration-1}
\left|\sum_{i=1}^n w_i(x_i^\trp u)^2 - \bE\left[\sum_{i=1}^n w_i(x_i^\trp u)^2\right]\right| \le 
O\left(\sqrt{\log(1/\delta')\sum_{i=1}^n w_i^2}
+ \log(1/\delta')\max_i |w_i|\right).
\end{equation}
By a union bound, with probability at least $1-\delta$, the above inequality holds for every $u\in O'$.
By the definition of $O'$,
\begin{equation}
\label{eq:proof-matrix-concentration-2}
\|E - \bar E\| \le O\left(\sup_{u\in O'}|u(E - \bar E)u|\right) = O\left(\sup_{u\in O'}\left|\sum_{i=1}^n w_i(x_i^\trp u)^2 - \bE\left[\sum_{i=1}^n w_i(x_i^\trp u)^2\right]\right|\right).
\end{equation}
Combining \eqref{eq:proof-matrix-concentration-1} and \eqref{eq:proof-matrix-concentration-2} and noting that $\log(1/\delta') = O(d + \log(1/\delta))$ completes the proof.
\end{proof}
\begin{lemma}
\label{lm:matrix-norm}
Suppose $x_1,\ldots,x_n\in\bR^d$ are independent random vectors, and each $x_i$ is mean-zero and $1$-sub-Gaussian. Suppose $b_1,\ldots,b_n\in\bR$ are fixed real numbers. Then for any $\delta\in (0,1/2)$, with probability at least $1-\delta$,
\[
\|\begin{bmatrix}x_1 & \cdots & x_n\end{bmatrix}\diag(b_1,\ldots,b_n)\| \le O\left(\sqrt{\sum_{i=1}^n b_i^2 + (d + \log(1/\delta))\max_ib_i^2}\right).
\]
\end{lemma}
\begin{proof}
Define $F = \begin{bmatrix}x_1 & \cdots & x_n\end{bmatrix}\diag(b_1,\ldots,b_n)$ and $E = FF^\trp = \sum_{i=1}^n b_i^2 x_ix_i^\trp$. By  \Cref{lm:matrix-concentration}, with probability at least $1-\delta$,
\begin{equation}
\label{eq:proof-matrix-norm-1}
\|E - \bE[E]\| \le O\left(\sqrt{(d + \log(1/\delta))\sum_{i=1}^n b_i^4} + (d + \log(1/\delta))\max_ib_i^2\right).
\end{equation}
For every unit vector $u\in \bR^d$, $x_i^\trp u$ is $1$-sub-Gaussian, and thus $\bE[(x_i^\trp u)^2] \le O(1)$. This implies that $u\bE[x_ix_i^\trp] u \le O(1)$ for every unit vector $u\in \bR^d$ and thus $\|\bE[x_ix_i^\trp]\| \le O(1)$. Now we have
\begin{equation}
\label{eq:proof-matrix-norm-2}
\|\bE[E]\| \le \sum_{i=1}^n b_i^2\|x_ix_i^\trp\| \le O\left(\sum_{i=1}^n b_i^2\right).
\end{equation}
Combining \eqref{eq:proof-matrix-norm-1} and \eqref{eq:proof-matrix-norm-2} and using the fact that $\sum_{i=1}^n b_i^4 \le \left(\sum_{i=1}^n b_i^2\right)\max_ib_i^2$, with probability at least $1-\delta$,
\[
\|E\| \le \|E - \bE[E]\| + \|\bE[E]\| \le O\left(\sum_{i=1}^n b_i^2 + (d + \log(1/\delta))\max_ib_i^2\right).
\]
The lemma is proved by $\|F\| = \sqrt{\|E\|}$.
\end{proof}
\begin{lemma}
\label{lm:high-prob-noise-fixed}
Let $X\in\bR^{d\times m}$ be a fixed matrix and let $v\in\bR^m$ be a mean zero random vector with sub-Gaussian constant $1$. For $\delta\in (0,1/2)$, with probability at least $1-\delta$,
\[
\|Xv\|_2 \le O(\|X\|\sqrt{r + \log(1/\delta)}),
\]
where $r$ is the rank of $X$.
\end{lemma}
\begin{proof}
By the singular value decomposition, there exists $P\in \Ob_d$ such that all but the first $r$ rows of $PX$ are zeros. Let $Y\in \bR^{r\times m}$ denote the first $r$ rows of $PX$. We have 
\begin{equation}
\label{eq:high-prob-noise-fixed-1}
\|Xv\|_2 = \|PXv\|_2 = \|Yv\|_2.
\end{equation}

Let $O'$ be the $O'_r$ in \Cref{lm:covering}. For every $u\in O'$, $u^\trp Yv$ has sub-Gaussian constant $O(\|u^\trp Y\|_2)$.
By \Cref{lm:sub-gaussian}, for $\delta' = \delta/|O'|$, with probability at least $1-\delta'$,
\[
|u^\trp Yv| \le O\left(\|u^\trp Y\|_2\sqrt{\log(1/\delta')}\right) \le O\left(\|Y\|\sqrt{\log(1/\delta')}\right).
\]
By a union bound, with probability at least $1-\delta$,
\begin{equation}
\label{eq:high-prob-noise-fixed-2}
\sup_{u\in O'}|u^\trp Yv| \le O\left(\|u^\trp Y\|_2\sqrt{\log(1/\delta')}\right) \le O\left(\|Y\|\sqrt{\log(1/\delta')}\right).
\end{equation}
By the definition of $O'$,
\begin{equation}
\label{eq:high-prob-noise-fixed-3}
\|Yv\|_2 \le O\left(\sup_{u\in O'}|u^\trp Yv|\right).
\end{equation}
The lemma is proved by combining \eqref{eq:high-prob-noise-fixed-1}, \eqref{eq:high-prob-noise-fixed-2} and \eqref{eq:high-prob-noise-fixed-3} and noting that $\|Y\| = \|X\|$ and $\log(1/\delta') = O(r + \log(1/\delta))$.
\end{proof}
\begin{lemma}
\label{lm:extended-matrix-bernstein}
Let $Z_1,\ldots,Z_n\in\bR^{d_1\times d_2}$ be mean-zero independent random matrices. Suppose for real numbers $R,p_1,\ldots,p_n\in\bR_{\ge 0}$, we have $\Pr[\|Z_i\| \ge R] \le p_i$. Moreover,
\[
\max\{\|\bE [Z_iZ_i^\trp]\|, \|\bE[Z_i^\trp Z_i]\|\} \le \sigma^2_i.
\]
Then with probability at least $1 - 2(d_1 + d_2)e^{-t} - \sum_{i=1}^n p_i$,
\[
\left\|\sum_{i=1}^n Z_i\right\|\le O\left(\sqrt{\sum_{i=1}^n\sigma_i^2t} + Rt\right).
\]
\end{lemma}
\begin{proof}
Since probabilities are always nonnegative, the lemma is trivial if $\sum_{i=1}^n p_i \ge 1$ or if $t \le c$ for a sufficiently small positive constant $c$. We thus assume $\sum_{i=1}^n p_i \le 1$ and $t \ge \Omega(1)$.

Define $Z_i' = Z_i\ind(\|Z_i\|\le R)$. For every unit vector $u\in\bR^d$, by Cauchy-Schwarz,
\[
\bE[\|Z_iu\|_2\ind(\|Z_i\|\ge R)]^2 \le \bE[\|Z_iu\|_2^2]\bE[\ind(\|Z_i\|\ge R)^2] \le \sigma_i^2 \Pr[\|Z_i\|\ge R] \le \sigma_i^2p_i.
\]
By Jenson's inequality,
\[
\bE[\|Z_iu\|_2\ind(\|Z_i\|\ge R)] \ge \|\bE[Z_iu\ind(\|Z_i\|\ge R)]\|_2 = \|\bE[Z_i\ind(\|Z_i\|\ge R)]u\|_2.
\]
Combining,
\[
\|\bE[Z_i\ind(\|Z_i\|\ge R)]u\|_2 \le \sigma_i\sqrt{p_i},\quad \text{for every unit vector}\ u\in\bR^{d_2},
\]
which implies that $\|\bE[Z_i']\| = \|\bE[Z_i] - \bE[Z_i']\| = \|\bE[Z_i\ind(\|Z_i\| > R)]\| \le \sigma_i \sqrt{p_i}$.

We apply the matrix Bernstein inequality \citep[see][Exercise 5.4.15]{vershynin2018high} to $Z_i'$. For every $i$, $\|Z_i' - \bE[Z_i']\| \le \|Z_i'\| + \|\bE[Z_i']\| \le \|Z_i'\| + \bE[\|Z_i'\|] \le 2R$, and
\[
\bE[(Z_i' - \bE[Z_i'])(Z_i' - \bE[Z_i'])^\trp] = \bE[Z_i'(Z_i')^\trp] - \bE[Z_i']\bE[Z_i']^\trp \preceq \bE[Z_i'(Z_i')^\trp] \preceq \sigma_i^2 I.
\]
The matrix Bernstein inequality implies that with probability at least $1 - 2(d_1 + d_2)e^{-t}$,
\[
\left\|\sum_{i=1}^n Z_i' - \sum_{i=1}^n\bE[Z_i']\right\| \le O\left(\sqrt{\sum_{i=1}^n \sigma_i^2 t} + Rt\right).
\]
By the union bound, with probability at least $1 - 2(d_1 + d_2)e^{-t} - \sum_{i=1}^n p_i$,
\[
\left\|\sum_{i=1}^n Z_i - \sum_{i=1}^n\bE[Z_i']\right\| \le O\left(\sqrt{\sum_{i=1}^n \sigma_i^2 t} + Rt\right),
\]
in which case,
\begin{align*}
\left\|\sum_{i=1}^n Z_i\right\| & \le \left\|\sum_{i=1}^n\bE[Z_i']\right\| + O\left(\sqrt{\sum_{i=1}^n \sigma_i^2 t} + Rt\right)\\
& \le \sum_{i=1}^n \sigma_i\sqrt{p_i} + O\left(\sqrt{\sum_{i=1}^n \sigma_i^2 t} + Rt\right)\\
& \le \sqrt{\left(\sum_{i=1}^n \sigma_i^2\right)\left(\sum_{i=1}^n p_i\right)} + O\left(\sqrt{\sum_{i=1}^n \sigma_i^2 t} + Rt\right).\tag{by Cauchy-Schwarz}
\end{align*}
The lemma is proved by our assumption that $\sum_{i=1}^n p_i \le 1$ and $t \ge \Omega(1)$.
\end{proof}
\section{Fano's Inequality and Local Packing Numbers}
\label{sec:fano}
For two probability distributions $p_1,p_2$ over the same set $\cX$ with probability densities $p_1(x)$ and $p_2(x)$, their KL divergence is defined to be
\[
\kl(p_1\|p_2) = \bE_{X\sim p_1}\left[\log\frac{p_1(X)}{p_2(X)}\right].
\]
For a joint distribution $p$ over random variables $X$ and $Y$, we define the mutual information between $X$ and $Y$ to be
\[
I(X;Y) = \kl(p\|p_X\times p_Y),
\]
where $p_X$ (resp.\ $p_Y$) is the marginal distribution of $X$ (resp.\ $Y$), and $p_X\times p_Y$ is the product distribution of $p_X$ and $p_Y$.
The following claim is standard:
\begin{claim}
For random variables $X,Y,Z$, if $Y$ and $Z$ are conditionally independent given $X$, then
\begin{align}
I(X;Y,Z) & \le I(X;Y) + I(X;Z),\label{eq:mutual-information-1}\\
I(Y;X,Z) & = I(Y;X).\label{eq:mutual-information-2}
\end{align}
\end{claim}
\begin{theorem}[Fano's Inequality]
\label{thm:fano}
Let $X, Y, \hat X$ be random objects. Assume their joint distribution forms a Markov chain $X \rightarrow Y \rightarrow \hat X$. Assume that the marginal distribution of $X$ is uniform over a finite set $\cX$. Then,
\[
\Pr[\hat X \ne X] \ge 1 - \frac{I(X;Y) + \log 2}{\log |\cX|}.
\]
\end{theorem}
\begin{lemma}
\label{lm:mutual-information-kl}
In the setting of \Cref{thm:fano}, let $P_x$ denote the conditional distribution of $Y$ given $X = x$. Then,
\[
I(X;Y) \le \max_{x,x'\in \cX}\kl(P_{x'}\|P_x).
\]
\end{lemma}
\begin{proof}
Let $P_x(y)$ denote $\Pr[Y = y|X = x]$.
Since $X$ distributes uniformly over $\cX$, the marginal distribution of $Y$ is given by $\Pr[Y = y] = \sum_{x\in \cX}P_x(y)/|\cX|$. By the definition of $I(X;Y)$,
\begin{align*}
I(X;Y) & = \bE_X\bE_{Y\sim P_X}\left[\log\frac{P_X(Y)/|\cX|}{\sum_{x\in \cX}P_{x}(Y)/|\cX|^2}\right]\\
& = \bE_X\bE_{Y\sim P_X}\left[\log\frac{P_X(Y)}{\sum_{x\in \cX}P_{x}(Y)/|\cX|}\right]\\
& \le \bE_X\bE_{Y\sim P_X}\left[\frac{1}{|\cX|}\sum_{x\in \cX}\log\frac{P_X(Y)}{P_x(Y)}\right]\tag{by Jensen's Inequality}\\
& = \frac 1{|\cX|} \sum_{x'\in \cX} \frac 1{|\cX|} \sum_{x\in\cX}\bE_{Y\sim P_{x'}}\left[\log\frac{P_{x'}(Y)}{P_{x}(Y)}\right]\\
& = \frac{1}{|\cX|^2}\sum_{x\in\cX}\sum_{x'\in\cX}\kl(P_{x'}\|P_x)\\
& \le \max_{x,x'\in\cX}\kl(P_{x'}\|P_x).
\qedhere
\end{align*}
\end{proof}
\begin{lemma}[{\citep[see][Lemma 1]{cai2013sparse}}]
\label{lm:packing}
There exists an absolute constant $C>0$ with the following property.
For any positive integers $k \le d$ and real number $t \in (0, 1/C]$, there exists a subset $O'\subseteq \Ob_{d,k}$ with size at least $10^{k(d-k)}$ such that
\begin{enumerate}
\item For distinct $U,V\in O'$, $\|UU^\trp - VV^\trp\|_F > t$.
\item For any $U,V\in O'$, $\|UU^\trp - VV^\trp\|_F \le Ct$.
\end{enumerate}
\end{lemma}
\section{Proof of Theorem~\ref{thm:DK}}
\thmDK*
\label{sec:proof-DK}
\begin{proof}
We assume $\frac{2\|A - \hat A\|}{\lambda_k - \lambda_{k+1}} < 1$ because otherwise the theorem is trivial.

Let $\hat \lambda_i$ denote the $i$-th largest eigenvalue of $\hat A$. 
By Weyl's inequality, 
\[
\hat \lambda_{k+1}  - \lambda_{k+1} \le \|A - \hat A\| \le \frac{\lambda_k - \lambda_{k+1}}{2},
\]
where we used our assumption $\frac{2\|A - \hat A\|}{\lambda_k - \lambda_{k+1}} < 1$ in the last inequality. This implies
\begin{equation}
\label{eq:proof-DK-1}
\lambda_k - \hat \lambda_{k+1} \ge \frac{\lambda_k - \lambda_{k+1}}{2} > 0.
\end{equation}
By the Davis-Kahan $\sin\theta$ theorem \citep{davis1970rotation},
\begin{equation}
\label{eq:proof-DK-2}
\sin \theta \le \frac{\|A - \hat A\|}{\lambda_k - \hat \lambda_{k+1}}.
\end{equation}
Combining \eqref{eq:proof-DK-1} and \eqref{eq:proof-DK-2} completes the proof.
\end{proof}

\section{Proof of Theorem~\ref{thm:PCA-fixed}}
\label{sec:proof-PCA-fixed}
We recall:
\thmPCAfixed*
By \eqref{eq:PCA-DK}, it suffices to upper bound the spectral norm of $A - \sum_{i=1}^n w_i\mu_i\mu_i^\trp$. We achieve this goal by proving \Cref{lm:PCA-decomposition-1,lm:PCA-decomposition-2,lm:PCA-decomposition-3} in which we bound the spectral norm of each term on the right-hand-side of the following decomposition:
\begin{equation}
\label{eq:PCA-decomposition}
A - \sum_{i=1}^n w_i \mu_i\mu_i^\trp =  \sum_{i=1}^n w_i \mu_i\bar z_i^\trp + \sum_{i=1}^n w_i\bar z_i \mu_i^\trp + \sum_{i=1}^n \frac{w_i}{m_i(m_i - 1)}\sum_{j_1\ne j_2}z_{ij_1}z_{ij_2},
\end{equation}
where $\bar z_i = \frac 1{m_i}\sum_{j=1}^{m_i}z_{ij}$. 
\begin{lemma}
\label{lm:PCA-decomposition-1}
In the setting of \Cref{thm:PCA-fixed}, for any $\delta\in (0,1/2)$, with probability at least $1-\delta$,
\[
\|\sum_{i=1}^n w_i\mu_i\bar z_i^\trp\| \le O\left(\xi\sqrt{d + \log(1/\delta)}\right).
\]
\end{lemma}
\begin{proof}
Let $O'$ denote the set $O'_d\subseteq \Ob_{d,1}$ guaranteed by \Cref{lm:covering}. For any fixed $u,v\in O'$, by \Cref{lm:sub-Gaussian-sum}, $\bar z_i^\trp v$ is $\eta_i/\sqrt{m_i}$-sub-Gaussian, and thus
$w_i(\mu_i^\trp u)(\bar z_i^\trp v)$ is $w_i|\mu_i^\trp u|\eta_i/\sqrt{m_i}$-sub-Gaussian. 
By \Cref{lm:sub-gaussian}, for any $\delta'\in (0,1/2)$, with probability at least $1-\delta'$,
\begin{equation}
\label{eq:PCA-decomposition-1-1}
\left|\sum_{i=1}^n w_i(\mu_i^\trp u)(\bar z_i^\trp v)\right| \le O\left(\sqrt{\log(1/\delta')\sum_{i=1}^nw_i^2(\mu_i^\trp u)^2\eta_i^2/m_i}\right).
\end{equation}
By the definition of $\xi$,
\[
\sum_{i=1}^nw_i^2(\mu_i^\trp u)^2\eta_i^2/m_i = u^\trp \left(\sum_{i=1}^nw_i^2\mu_i\mu_i^\trp\eta_i^2/m_i\right)u 
\le \left\| \sum_{i=1}^nw_i^2\mu_i\mu_i^\trp\eta_i^2/m_i\right\| = \xi^2.
\]
Plugging this into \eqref{eq:PCA-decomposition-1-1},
and setting $\delta' = \delta/|O'|^2$, with probability at least $1 - \delta/|O'|^2$,
\[
\left|\sum_{i=1}^n w_i(\mu_i^\trp u)(\bar z_i^\trp v)\right| \le O\left(\xi \sqrt{\log(1/\delta')}\right) \le O\left(\xi\sqrt{d + \log(1/\delta)}\right),
\]
where we used the fact that $\log |O'| = O(d)$ to obtain the last inequality.
By a union bound over $u,v\in O'$, with probability at least $1-\delta$, $\sup_{u\in O'}\sup_{v\in O'}|u^\trp \left(\sum_{i=1}^n w_i \mu_i\bar z_i^\trp\right) v| \le O(\xi\sqrt{d + \log(1/\delta)}$. The lemma is proved by combining this with $\left\|\sum_{i=1}^n w_i \mu_i\bar z_i^\trp\right\| \le \sup_{u\in O'}\sup_{v\in O'}|u^\trp \left(\sum_{i=1}^n w_i \mu_i\bar z_i^\trp\right) v|$ by \Cref{lm:covering}.
\end{proof}
\begin{lemma}
\label{lm:PCA-decomposition-2}
In the setting of \Cref{thm:PCA-fixed}, 
define $Z = \sum_{i=1}^n \frac{w_i}{m_i(m_i - 1)}(\sum_{j=1}^{m_i} z_{ij})(\sum_{j=1}^{m_i} z_{ij})^\trp$.
For any $\delta\in (0,1/2)$,
with probability at least $1-\delta$,
\begin{align}
\left\|Z - \bE[Z]\right\|& \notag\\
\le {} & O\left(\sqrt{(d + \log(1/\delta))\sum_{i=1}^n w_i^2\eta_i^4/m_i^2}
+ (d + \log(1/\delta))\max_i w_i\eta_i^2/m_i
\right).\label{eq:PCA-decomposition-2}
\end{align}
\end{lemma}
\begin{proof}
By \Cref{lm:sub-Gaussian-sum}, $\sum_{j=1}^{m_i}z_{ij}$ has sub-Gaussian constant $b_i$ for some $b_i\in\bR_{\ge 0}$ satisfying $b_i = O(\sqrt{m_i}\eta_i)$. Define $\hat z_i = \frac{1}{b_i}\sum_{j=1}^{m_i}z_{ij}$ and now $\hat z_i$ has sub-Gaussian constant $1$. We write $Z$ in terms of $\hat z_i$:
\[
Z = \sum_{i=1}^n \frac{w_ib_i^2}{m_i(m_i - 1)}\hat z_i\hat z_i^\trp.
\]
Applying \Cref{lm:matrix-concentration} completes the proof.
\end{proof}
\begin{lemma}
\label{lm:PCA-decomposition-3}
In the setting of \Cref{thm:PCA-fixed}, 
define $Z = \sum_{i=1}^n \frac{w_i}{m_i(m_i - 1)}\sum_{j=1}^{m_i}z_{ij}z_{ij}^\trp$.
For any $\delta\in(0,1/2)$,
with probability at least $1-\delta$,
\begin{align}
& \left\|Z - \bE[Z]\right\| \notag \\
\le {} & O\left(\sqrt{(d + \log(1/\delta))\sum_{i=1}^n w_i^2\eta_i^4/m_i^3}
+ (d + \log(1/\delta))\max_i w_i\eta_i^2/m_i^2
\right).\label{eq:PCA-decomposition-3}
\end{align}
\end{lemma}
\begin{proof}
Define $\hat z_{ij} = \frac{1}{\eta_i}z_{ij}$ and now $\hat z_{ij}$ has sub-Gaussian constant $1$. We write $Z$ in terms of $\hat z_{ij}$:
\[
Z = \sum_{i=1}^n \sum_{j=1}^{m_i}\frac{w_i\eta_i^2}{m_i(m_i - 1)}\hat z_{ij}\hat z_{ij}^\trp.
\]
Applying \Cref{lm:matrix-concentration} completes the proof.
\end{proof}
We are now ready to finish the proof of \Cref{thm:PCA-fixed}.
\begin{proof}[Proof of \Cref{thm:PCA-fixed}]
We start with the last term in \eqref{eq:PCA-decomposition}:
\begin{align*}
& \sum_{i=1}^n\frac{w_i}{m_i(m_i - 1)}\sum_{j_1\ne j_2}z_{ij_1}z_{ij_2}\\
= {} & 
\sum_{i=1}^n\frac{w_i}{m_i(m_i - 1)}\left(\sum_{j=1}^{m_i}z_{ij}\right)\left(\sum_{j=1}^{m_i}z_{ij}\right)^\trp - \sum_{i=1}^n \frac{w_i}{m_i(m_i - 1)}\sum_{j=1}^{m_i}z_{ij}z_{ij}^\trp.
\end{align*}
It is clear that $\bE\left[\sum_{i=1}^n\frac{w_i}{m_i(m_i - 1)}\sum_{j_1\ne j_2}z_{ij_1}z_{ij_2}\right] = 0$, so combining \Cref{lm:PCA-decomposition-2,lm:PCA-decomposition-3}, with probability at least $1-\delta$,
\begin{equation}
\label{eq:PCA-fixed-proof-1}
\left\|\sum_{i=1}^n\frac{w_i}{m_i(m_i - 1)}\sum_{j_1\ne j_2}z_{ij_1}z_{ij_2}\right\| \le O\left(\sqrt{(d + \log(1/\delta))\sum_{i=1}^n w_i^2\eta_i^4/m_i^2}
+ (d + \log(1/\delta))\max_i w_i\eta_i^2/m_i\right),
\end{equation}
where we used the fact that the right-hand-side of \eqref{eq:PCA-decomposition-2} is at most the right-hand-side of \eqref{eq:PCA-decomposition-3} (up to a constant factor).

By \Cref{lm:PCA-decomposition-1},
\begin{equation}
\label{eq:PCA-fixed-proof-2}
\left\|\sum_{i=1}^n w_i \mu_i\bar z_i^\trp + \sum_{i=1}^n w_i\bar z_i \mu_i^\trp\right\| \le 2\left\|\sum_{i=1}^n w_i \mu_i\bar z_i^\trp\right\|
\le 
O\left(\xi\sqrt{d + \log(1/\delta)}\right).
\end{equation}
Plugging \eqref{eq:PCA-fixed-proof-1} and \eqref{eq:PCA-fixed-proof-2} into \eqref{eq:PCA-decomposition} and then into \eqref{eq:PCA-DK} proves the theorem.
\end{proof}
\section{Proof of Corollary~\ref{cor:PCA-fixed}}
\label{sec:proof-cor-PCA-fixed}
\begin{proof}
By the definition of $\xi^2$ in \Cref{thm:PCA-fixed}, $\xi^2 = \|\sum_{i=1}^n \frac{\mu_i\mu_i^\trp t^2}{n^2}\|  \le \frac{t^2\sigma_1^2}{n}$.
Plugging this into \eqref{eq:PCA-fixed}, we know that the following inequality holds with probability at least $1-\delta$:
\begin{align}
\sin\theta & \le O\left(\sigma_k^{-2}\sqrt{(d + \log(1/\delta))\left(\frac{t^2\sigma_1^2}{n} + \frac{t^4}{n}\right)}
+
\sigma_k^{-2}(d + \log(1/\delta))\frac{t^2}{n}
\right)\notag \\
& = O\left(\frac{t\sigma_1 + t^2}{\sigma_k^2}\sqrt{\frac{d + \log(1/\delta)}{n}} +
\sigma_k^{-2}(d + \log(1/\delta))\frac{t^2}{n}
 \right).
\label{eq:PCA-fixed-2}
\end{align}
Since $\sin\theta \le 1$ always holds, inequality \eqref{eq:PCA-fixed-2} implies
\begin{equation}
\label{eq:PCA-fixed-3}
\sin\theta  \le O\left(\frac{t\sigma_1 + t^2}{\sigma_k^2}\sqrt{\frac{d + \log(1/\delta)}{n}} +
\min\left\{1,\sigma_k^{-2}(d + \log(1/\delta))\frac{t^2}{n}\right\}
 \right)
\end{equation}
Since $\min\{1,y\} \le \sqrt y$ for every $y\in\bR_{\ge 0}$, we have 
\begin{equation}
\label{eq:PCA-fixed-4}
\min\left\{1, \sigma_k^{-2}(d + \log(1/\delta))\frac{t^2}{n}\right\} \le \frac{t}{\sigma_k}\sqrt{\frac{(d + \log(1/\delta)}{n}} \le \frac{t\sigma_1}{\sigma_k^2}\sqrt{\frac{(d + \log(1/\delta)}{n}}.
\end{equation}
Plugging \eqref{eq:PCA-fixed-4} into \eqref{eq:PCA-fixed-3} proves the corollary.
\end{proof}
\section{Proof of Claim~\ref{claim:sigma-k}}
\label{sec:proof-sigma-k}
\begin{proof}
Without loss of generality, we can assume that the subspace $\Gamma$ in Assumption~\ref{assumption:mu} is the subspace containing all vectors with all but the first $k$ coordinates being zeros. We let $\hat \mu_i\in\bR^k$ denote the first $k$ coordinates of $\mu_i$. By \Cref{lm:matrix-concentration}, with probability at least $1-\delta/2$,
\begin{align}
& \left\|\sum_{i=1}^n w_i\hat\mu_i\hat\mu_i^\trp - \bE\left[\sum_{i=1}^n w_i\hat\mu_i\hat\mu_i^\trp\right]\right\|\notag \\
\le {} & O\left(\sigma^2\left(\sqrt{(k + \log(1/\delta))\sum_{i=1}^n w_i^2} + (k + \log(1/\delta))\max_i|w_i|\right)\right)\notag \\
\le {} & O\left(\sigma^2\left(\sqrt{\frac 1{C_*}} + \frac{1}{C_*}\right)\right),\label{eq:proof-sigma-k-1}
\end{align}
where we used the fact that $\sum_{i=1}^n w_i^2 \le \max_i w_i \le 1/C^*(k + \log(1/\delta))$.
Choosing $C_*$ large enough, \eqref{eq:proof-sigma-k-1} implies
\begin{equation}
\label{eq:proof-sigma-k-2}
\left\|\sum_{i=1}^n w_i\hat\mu_i\hat\mu_i^\trp - \bE\left[\sum_{i=1}^n w_i\hat\mu_i\hat\mu_i^\trp\right]\right\| \le \sigma^2/2.
\end{equation}
By Assumption~\ref{assumption:mu}, the $k$-th largest eigenvalue of $\bE[\mu_i\mu_i^\trp]$ is at least $\sigma^2$. This implies that the $k$-th largest eigenvalue of $\bE[\sum_{i=1}^n w_i\hat\mu_i\hat\mu_i^\trp]$ is at least $\sigma^2$. 
The claim is proved by
combining this with \eqref{eq:proof-sigma-k-2} and noting that $\sum_{i=1}^n w_i\hat \mu_i\hat \mu_i^\trp$ and $\sum_{i=1}^n w_i\mu_i\mu_i^\trp$ have the same $k$-th largest eigenvalue.
\end{proof}
\section{Proof of Theorem~\ref{thm:PCA-weights-chosen}}
\thmPCAweightschosen*
\label{sec:proof-PCA-weights-chosen}
\begin{proof}
By Assumption~\ref{assumption} and our choice of $w_i$, we have $w_i \le 1/C_*(k + \log(1/\delta))$. By the definition of $C_*$, with probability at least $1-\delta/2$, 
\begin{equation}
\label{eq:PCA-weights-chosen-proof-1}
\sigma_k^2 \ge \sigma^2/2. 
\end{equation}
By \Cref{lm:matrix-norm}, with probability at least $1-\delta/4$,
\begin{equation}
\label{eq:PCA-weights-chosen-proof-2}
\xi \le O\left(\sqrt{\sum_{i=1}^n \frac{w_i^2 \eta_i^2\sigma^2}{m_i} + (k + \log(1/\delta))\max_i\frac{w_i^2\eta_i^2\sigma^2}{m_i}}\right).
\end{equation}
Replacing $\delta$ in \Cref{thm:PCA-fixed} by $\delta/4$  and
plugging \eqref{eq:PCA-weights-chosen-proof-1} and \eqref{eq:PCA-weights-chosen-proof-2} into \eqref{eq:PCA-fixed}, by the union bound, with probability at least $1-\delta$,
\begin{align}
\sin\theta \le {} & O\Bigg(\sqrt{(d + \log(1/\delta))\sum_{i=1}^n \left(\frac{w_i^2\eta_i^2}{\sigma^2m_i} + \frac{w_i^2\eta_i^4}{\sigma^4m_i^2}\right)}\notag\\
& + \sqrt{(d + \log(1/\delta))(k + \log(1/\delta))\max_i\frac{w_i^2\eta_i^2}{\sigma^2m_i}}
 + (d + \log(1/\delta))\max_i \frac{w_i\eta_i^2}{\sigma^2m_i}. \Bigg)\label{eq:proof-PCA-weights-chosen-1}
\end{align}
By definition, $\frac{\eta_i^2}{\sigma^2m_i} + \frac{\eta_i^4}{\sigma^4m_i^2} = 1/\gamma_i \le 1/\gamma_i'$. Therefore, inequality \eqref{eq:proof-PCA-weights-chosen-1} implies
\begin{align*}
& \sin\theta\\
\le {} & O\left(\sqrt{(d + \log(1/\delta))\sum_{i=1}^n \frac{w_i^2}{\gamma_i'}}
 + \sqrt{(d + \log(1/\delta))(k + \log(1/\delta))\max_i\frac{w_i^2}{\gamma_i'}}
 + (d + \log(1/\delta))\max_i \frac{w_i}{\gamma_i'}\right).
\end{align*}
Plugging $w_i = \frac{\gamma_i'}{\sum_{\ell = 1}^n \gamma_\ell'}$ into the inequality above,
\begin{equation}
\label{eq:proof-PCA-weights-chosen-2}
\sin\theta
\le  O\left(\sqrt{\frac{d + \log(1/\delta)}{\sum_{i=1}^n \gamma_i'}} + \sqrt{\frac{(d + \log(1/\delta))(k + \log(1/\delta))\max_i \gamma_i'}{\left(\sum_{i=1}^n\gamma_i'\right)^2}} + \frac{d + \log(1/\delta)}{\sum_{i=1}^n\gamma_i'}\right).
\end{equation}
Since $\sin\theta \le 1$ always holds, inequality \eqref{eq:proof-PCA-weights-chosen-2} implies
\begin{align*}
\sin\theta
\le {} & O\left(\sqrt{\frac{d + \log(1/\delta)}{\sum_{i=1}^n \gamma_i'}} + \sqrt{\frac{(d + \log(1/\delta))(k + \log(1/\delta))\max_i \gamma_i'}{\left(\sum_{i=1}^n\gamma_i'\right)^2}} + \min\left\{1,\frac{d + \log(1/\delta)}{\sum_{i=1}^n\gamma_i'}\right\}\right).
\end{align*}
The theorem is proved by the following facts:
\begin{align*}
(k + \log(1/\delta))\max_i \gamma_i' & \le O \left(\sum_{i=1}^n \gamma_i'\right), \tag{by Assumption~\ref{assumption}}\\ \min\left\{1, \frac{d + \log(1/\delta)}{\sum_{i=1}^n \gamma_i'}\right\} & \le \sqrt{\frac{d + \log(1/\delta)}{\sum_{i=1}^n\gamma_i'}}.
\qedhere
\end{align*}
\end{proof}
\section{Proof of Lemma~\ref{lm:kl-angle}}
\label{sec:proof-kl-angle}
\begin{proof}
The KL divergence between mean-zero Gaussians can be computed by the following formula:
\begin{equation}
\label{eq:proof-kl-angle-0}
\kl(N(0, \hat\Sigma)\|N(0, \Sigma))  = \frac 12\left(\tr(\Sigma^{-1}\hat\Sigma) - d + \log\frac{\det\Sigma}{\det\hat\Sigma}\right).
\end{equation}
It is clear that 
\begin{equation}
\label{eq:proof-kl-angle-1}
\det\Sigma = \det\hat\Sigma = (\sigma^2  + \eta^2)^k (\eta^2)^{d-k}. 
\end{equation}
It remains to compute $\tr(\Sigma^{-1}\hat\Sigma) - d$.

Define $J = \begin{bmatrix}I_k & 0 \\ 0 & 0\end{bmatrix}\in\bR^{d\times d}$.
By the definition of principal angles in \Cref{sec:preliminaries}, there exists $P,\hat P\in \Ob_d$ such that
$UU^\trp = PJP^\trp, \hat U \hat U^\trp = \hat P J\hat P^\trp$ and 
\[
P^\trp \hat P = 
\begin{bmatrix}
I_{k - \ell} & 0 & 0 & 0 \\
0 & \cos\Theta & \sin\Theta & 0\\
0 & -\sin\Theta & \cos\Theta & 0\\
0 & 0 & 0 & I_{d - \hat k - \ell}
\end{bmatrix},
\]
where $\theta_1,\ldots,\theta_\ell$ are the principal angles between $\col(U)$ and $\col(\hat U)$, $\cos\Theta = \diag(\cos\theta_1,\ldots,\cos\theta_\ell)$, and $\sin\Theta = \diag(\sin\theta_1,\ldots,\sin\theta_\ell)$.
Now we have
$\hat U\hat U^\trp = \hat P J\hat P^\trp = P(P^\trp \hat P) J (P^\trp \hat P)^\trp P^\trp = P \hat J P^\trp$ where 
\[
\hat J = (P^\trp \hat P) J (P^\trp \hat P)^\trp = 
\begin{bmatrix}I_{k - \ell} & 0 & 0 & 0 \\ 0 & (\cos\Theta)^2 & -\cos\Theta\sin\Theta & 0 \\ 0& -\cos\Theta\sin\Theta & (\sin\Theta)^2 & 0\\ 0 & 0 & 0 & 0\end{bmatrix}.
\]
Therefore, $\Sigma = P(\sigma^2 J + \eta^2 I)P^\trp$ and $\hat\Sigma = P(\sigma^2\hat J + \eta^2 I)P^\trp$, which implies
\[
\Sigma^{-1}\hat \Sigma = P(\sigma^2 J + \eta^2 I)^{-1}(\sigma^2 \hat J + \eta^2 I)P^\trp.
\]
Since $\sigma^2J + \eta^2I = \diag(\underbrace{\sigma^2 + \eta^2,\ldots,\sigma^2 + \eta^2}_{k}, \underbrace{\eta^2,\ldots,\eta^2}_{d - k})$ and the diagonal entries of $\sigma^2\hat J + \eta^2 I$ are
\begin{align*}
\underbrace{\sigma^2 + \eta^2,\ldots,\sigma^2  + \eta^2}_{k - \ell}, \underbrace{\sigma^2\cos^2\theta_1 + \eta^2,\ldots, \sigma^2\cos^2\theta_\ell + \eta^2}_{\ell}, &\\
\underbrace{\sigma^2\sin^2\theta_1 + \eta^2,\ldots, \sigma^2\sin^2\theta_\ell + \eta^2}_{\ell}, &
\underbrace{\eta^2, \ldots,\eta^2}_{d - k - \ell},
\end{align*}
we have
\begin{align}
\tr(\Sigma^{-1}\hat \Sigma) - d & = \tr((\sigma^2 J + \eta^2 I)^{-1}(\sigma^2 \hat J + \eta^2 I)) - d\notag \\
& = \sum_{i=1}^\ell \left(\frac{\sigma^2\cos^2\theta_i + \eta^2}{\sigma^2 + \eta^2} - 1\right) + \sum_{i=1}^\ell\left(\frac{\sigma^2\sin^2\theta_i + \eta^2}{\eta^2} - 1\right)\notag \\
& = \sum_{i=1}^\ell \sin^2\theta_i \left(\frac{\sigma^2}{\eta^2} - \frac{\sigma^2}{\sigma^2 + \eta^2}\right)\notag \\
& = \frac{\sigma^4\|UU^\trp - \hat U \hat U^\trp\|_F^2}{2(\eta^2\sigma^2 + \eta^4)}.\label{eq:proof-kl-angle-2}
\end{align}
Plugging \eqref{eq:proof-kl-angle-1} and \eqref{eq:proof-kl-angle-2} into \eqref{eq:proof-kl-angle-0} completes the proof.
\end{proof}
\section{Proof of Theorem~\ref{thm:PCA-lower}}
\label{sec:proof-PCA-lower}
We recall Theorem~\ref{thm:PCA-lower}:

\pcalb*

\begin{proof}
The theorem is trivial if $t$ is lower bounded by a positive absolute constant, so
without loss of generality, we can assume that $t \le 1/(2\sqrt 2C)$ for the constant $C$ in \Cref{lm:packing}. 

Now we describe the distribution of $E$ on which we prove the error lower bound \eqref{eq:PCA-lower} that does not depend on $\gamma_1,\ldots,\gamma_{k-1}$. As we mentioned in \Cref{sec:PCA-lower}, the first $k-1$ columns of $E$ are fixed to be $\begin{bmatrix}I_{k-1} \\ 0\end{bmatrix}$, so we only need to describe the distribution of the $k$-th column of $E$. By \Cref{lm:packing}, there exists $O'\subseteq \Ob_{d-k+1,1}$ with size at least $10^{d - k}$ such that
\begin{enumerate}
\item for distinct $u,v\in O'$, $\|uu^\trp - vv^\trp\|_F > 2\sqrt2t$;
\item for any $u,v\in O'$, $\|uu^\trp - vv^\trp\|_F \le O(t)$.
\end{enumerate}
The $k$-th column of $E$ is chosen by first drawing a uniform random $u\in O'$ and then prepend zeros to it. That is 
\begin{equation}
\label{eq:hard-distribution}
E = \begin{bmatrix}I_{k-1} & 0 \\ 0 & u\end{bmatrix}. 
\end{equation}
We let $\cE$ denote the support of the $E$ so that $E$ is chosen uniformly at random from $\cE$.

By the law of total expectation, conditioned on $W,V_1,(\mu_1,\ldots,\mu_{k-1})$ and $(x_{ij})_{1\le i < k, 1 \le j \le m_i}$ being fixed to some specific value, we still have $\Pr[\sin\theta \ge t] \le \delta$, where now $\Pr[\cdot]$ represents the conditional probability. We will abuse notation and omit explicitly writing out the conditioning throughout the proof.

After the conditioning, the randomness in $V$ comes only from the randomness in $E\in \cE$, and the distribution of $V$ is the uniform distribution over the set $\cV := \{WE: E\in \cE\}$. 
For any two matrices $V',V''\in \cV$, there exists $E',E''\in \cE$ such that $V' = WE'$ and $V'' = WE''$. Moreover, there exists $u,v\in O'$ such that $E' = \begin{bmatrix}I_{k-1} & 0 \\ 0 & u\end{bmatrix}$ and $E'' = \begin{bmatrix}I_{k-1} & 0 \\ 0 & v\end{bmatrix}$. This implies that
\begin{equation}
\label{eq:proof-PCA-lower-0-1}
\|V'(V')^\trp - V''(V'')^\trp\|_F = \|E'(E')^\trp - E''(E'')^\trp\|_F = \|uu^\trp - vv^\trp\|_F \le O(t).
\end{equation}
Moreover, when $V'\ne V''$, we have $u\ne v$ and thus
\begin{equation}
\label{eq:proof-PCA-lower-0-2}
\|V'(V')^\trp - V''(V'')^\trp\| = \|E'(E')^\trp - E''(E'')^\trp\| = \|uu^\trp - vv^\trp\|  > 2t,
\end{equation}
where we used the fact that $\|uu^\trp - vv^\trp\| = \sin\angle(u,v)$ and $2\sqrt{2}t< \|uu^\trp - vv^\trp\|_F = \sqrt{2\sin^2\angle(u,v)}$.
We define $\hat V\in \cV$ so that the maximum principal angle between $\col(\hat V)$ and $\hat \Gamma$ is minimized. When $\sin\theta \le t$ (i.e., $\sin\angle(\col(V), \Gamma) \le t$), for any $V'\in \cV$ different from $V$,
\[
\sin\angle(\col(V'), \Gamma) \ge \sin\angle(V,V') - \sin\angle(\col(V), \Gamma) > 2t - t = t \ge \sin\angle(\col(V), \Gamma),
\]
where we used $\sin\angle(V,V') = \|VV^\trp - V'(V')^\trp\| > 2t$ by \eqref{eq:proof-PCA-lower-0-2}.
Therefore, $\sin\theta \le t$ implies $\hat V = V$ and thus $\Pr[\sin\theta > t] \ge \Pr[\hat V \ne V]$.

Now we apply \Cref{thm:fano} to the following Markov chain
\[
V \rightarrow (x_{ij})_{k \le i \le n, 1\le j \le m_i} \rightarrow \hat V
\]
to get
\begin{equation}
\label{eq:PCA-lower-proof-1}
\delta \ge \Pr[\sin\theta > t] \ge \Pr[\hat V \ne V] \ge 1 - \frac{I(V;(x_{ij})_{k\le i\le n, 1\le j \le m_i}) +\log 2}{\log|\cV|}.
\end{equation}
Now we bound $I(V;(x_{ij})_{k \le i \le n, 1 \le j \le m_i})$. 
According to the graphical model in \Cref{fig:3}, $\mu_{k},\ldots,\mu_n$ are conditionally independent given $V$, and thus
$(x_{ij})_{1 \le j \le m_i}$ are conditionally independent for $ i = k,\ldots,n$ given $V$. By \eqref{eq:mutual-information-1},
\begin{equation}
\label{eq:PCA-lower-proof-1-0-1}
I(V;(x_{ij})_{k \le i \le n, 1 \le j \le m_i}) \le \sum_{i=k}^n I(V;(x_{ij})_{1 \le j \le m_i}).
\end{equation}

Let $\bar x_i$ denote $\frac 1{m_i}\sum_{j=1}^{m_i}x_{ij}$. Since $x_{ij}$ for $j = 1,\ldots,m_i$ are drawn iid from $N(\mu_i, \eta_i^2 I)$, it is a standard fact that $\bar x_i$ is a sufficient statistic for $\mu_i$, i.e., $(x_{ij})_{1\le j \le m_i}$ and $\mu_i$ are conditionally independent given $\bar x_i$. Therefore, for any measurable set $S\in(\bR^d)^{m_i}$,
\begin{equation}
\label{eq:proof-PCA-lower-1-1}
\Pr[(x_{ij})_{1\le j \le m_i}\in S|\mu_i,\bar x_i] = \Pr[(x_{ij})_{1\le j \le m_i}\in S|\bar x_i].
\end{equation}
Since $(x_{ij})_{j=1,\ldots,m_i}$ and $V$ are conditionally independent given $\mu_i$,
\begin{equation}
\label{eq:proof-PCA-lower-1-2}
\Pr[(x_{ij})_{1\le j \le m_i}\in S|V, \mu_i,\bar x_i] = \Pr[(x_{ij})_{1\le j \le m_i}\in S|\mu_i,\bar x_i].
\end{equation}
Combining \eqref{eq:proof-PCA-lower-1-1} and \eqref{eq:proof-PCA-lower-1-2}, $(x_{ij})_{1\le j\le m_i}$ and $V$ are conditionally independent given $\bar x_i$. By \eqref{eq:mutual-information-2},
\begin{equation}
\label{eq:lower-2}
I(V;(x_{ij})_{1\le j\le m_i}) = I(V;(x_{ij})_{1\le j\le m_i}, \bar x_i) = I(V;\bar x_i).
\end{equation}
Since $\bar x_i = \mu_i + \frac 1{m_i}\sum_{j=1}^{m_i}z_j$, the conditional distribution of $\bar x_i$ given $V$ is $N(0, \sigma^2 VV^\trp + \frac{\eta_i^2}{m_i}I)$.
By \Cref{lm:kl-angle} and inequality \eqref{eq:proof-PCA-lower-0-1},
\begin{equation}
\label{eq:lower-4}
I(V;\bar x_i) \le \sup_{V',V''\in \cV} \kl\left(N\left(0, \sigma^2 V'(V')^\trp + \frac{\eta^2_i}{m_i}I\right) \Big\| N\left(0,\sigma^2 V''(V'')^\trp + \frac{\eta^2_i}{m_i}I\right)\right) \le O(t^2 \gamma_i).
\end{equation}
Combining \eqref{eq:PCA-lower-proof-1-0-1}, \eqref{eq:lower-2}, and \eqref{eq:lower-4},
\[
I(V;(x_{ij})_{k \le i \le n, 1 \le j \le m_i}) \le O\left(t^2\sum_{i=k}^n\gamma_i\right).
\]
The theorem is proved by
plugging this into \eqref{eq:PCA-lower-proof-1} and noting that $\log |\cV| \ge 2(d - k)$.
\end{proof}

\section{Our Results in the Linear Models Setting}
\label{sec:linear-models-details}

In the linear models setting (see \Cref{sec:linear-model}),
our estimator is the subspace $\hat \Gamma$ spanned by the top-$k$ eigenvectors of $A$ defined in \eqref{eq:estimator-linear}. We use $\theta$ to denote the maximum principal angle between our estimator $\hat \Gamma$ and the true subspace $\Gamma$. 
To describe our guarantee on $\theta$, it is convenient to first define the following quantities for every $i = 1,\ldots,n$:
\[
p_i = \frac{(d + \log(n/\delta))(m_i + \log(n/\delta))}{m_i^2}, \quad
q_i = \sqrt{\frac{d + \log(n/\delta)}{m_i}} + p_i.
\]

We prove the following guarantees for our estimator in \Cref{sec:proof-linear-model,sec:proof-cor-linear-model}:
\begin{theorem}
\label{thm:linear-model}
There exists an absolute constant $C> 0$ such that for any $\delta\in (0,1/2)$,
with probability at least $1-\delta$,
\begin{align*}
\sin\theta \le {} & C\sigma_k^{-2}\sqrt{\sum_{i=1}^nw_i^2d\log(d/\delta)\left(\frac{\|\beta_i\|_2^4 + \eta_i^2\|\beta_i\|_2^2}{m_i} + \frac{\eta_i^4}{m_i^2}\right)}\\
& + C\sigma_k^{-2}\max_i w_i\log(d/\delta)(\|\beta_i\|_2^2q_i + \eta_i\|\beta_i\|_2 q_i + \eta_i^2p_i).
\end{align*}
\end{theorem}
\begin{corollary}
\label{cor:linear-model}
Suppose $m\in\bZ_{>0}$ and $\eta\in \bR_{>0}$ satisfy $2 \le m \le \min_{1\le i \le n}m_i$ and $\max_{1\le i\le n} \eta_i \le \eta$. Choosing $w_1 = \cdots = w_n = 1/n$, for any $\delta\in (0,1/2)$, with probability at least $1-\delta$,
\begin{equation}
\label{eq:cor-linear-model-1}
\sin \theta \le O\left(\log^3(nd/\delta)\left(\sqrt{\frac{d(\frac 1n\sum_{i=1}^n\|\beta_i\|_2^4 + \frac 1n\sum_{i=1}^n\eta^2\|\beta_i\|_2^2 + \eta^4/m)}{mn\sigma_k^4}} + \frac{d\max_i\|\beta_i\|_2^2}{mn\sigma_k^2}\right)\right).
\end{equation}
If we further assume that $\max_{1\le i \le n}\|\beta_i\|_2 \le r$, then \eqref{eq:cor-linear-model-1} implies
\begin{equation}
\label{eq:cor-linear-model-2}
\sin \theta \le O\left(\log^3(nd/\delta)\left(\sqrt{\frac{d(r^4 + r^2\eta^2 + \eta^4/m)}{mn\sigma_k^4}} \right)\right).
\end{equation}
\end{corollary}

In \citep{tripuraneni2021provable}, it is further assumed that $\eta = \Theta(1), r = \Theta(1),\|\beta_i\|_2 = \Theta(1),\delta = (mn)^{-100}$. Defining $\kappa = \frac 1n\sum_{i=1}^n \|\beta_i\|_2^2/k\sigma_k^2 = \Theta(1/k\sigma_k^2)$ as in \citep{tripuraneni2021provable} and using the fact $\sum_{i=1}^n \|\beta_i\|_2^4 \le \left(\max_i\|\beta_i\|_2^2\right) \sum_{i=1}^n \|\beta_i\|_2^2$, our bound becomes
\[
\sin \theta \le O\left(\sqrt{\frac{d}{mn\sigma_k^4}} \log^3(mn)\right)
= O\left(\sqrt{\frac{\kappa}{\sigma_k^2}\frac{dk}{mn}} \log^3(mn)\right),
\]
matching the bound in \citep[Theorem 3]{tripuraneni2021provable}.

As in the PCA setting, our proof of \Cref{thm:linear-model} is based on the Davis-Kahan $\sin\theta$ theorem (\Cref{thm:DK}) and a bound on the spectral norm of the difference between the matrix $A$ and its expectation. We prove and make crucial use of a generalization of the matrix Bernstein inequality (\Cref{lm:extended-matrix-bernstein}) which turns out to be slightly more convenient than a similar inequality used in \citep{tripuraneni2021provable}.
\subsection{Proof of Theorem~\ref{thm:linear-model}}
\label{sec:proof-linear-model}
Plugging $y_{ij} = x_{ij}^\trp \beta_i + z_{ij}$ into the definition of $A$,
\begin{align*}
A = \sum_{i=1}^n \frac{w_i}{m_i(m_i-1)}\sum_{j_1\ne j_2}x_{ij_1}(x_{ij_1}^\trp \beta_i + z_{ij_1})(\beta_i^\trp x_{ij_2} + z_{ij_2})x_{ij_2}^\trp.
\end{align*}
When $j_1\ne j_2$, conditioned on $x_{ij_1}$ and $x_{ij_2}$, the noise terms $z_{ij_1}$ and $z_{ij_2}$ are independent and have zero mean. Therefore,
\[
\bE[x_{ij_1}(x_{ij_1}^\trp \beta_i + z_{ij_1})(\beta_i^\trp x_{ij_2} + z_{ij_2})x_{ij_2}^\trp] = \bE[x_{ij_1}x_{ij_1}^\trp \beta_i\beta_i^\trp x_{ij_2}x_{ij_2}^\trp]  = \bE[\beta_i\beta_i^\trp],
\]
where we used the fact that $x_{ij_1}$ and $x_{ij_2}$ are independent and have zero mean and identity covariance matrix. Thus,
the expectation of $A$ is $\bar A = \sum_{i=1}^n w_i \beta_i\beta_i^\trp$, and as in the PCA setting, our goal is to bound the spectral norm of the difference $A - \bar A$. We decompose $A$ as
\begin{align}
A & = E + F + F^\trp + G,\mbox{ where}\label{eq:linear-model-decomposition}\\
E & = \sum_{i=1}^n \frac{w_i}{m_i(m_i-1)}\sum_{j_1\ne j_2}x_{ij_1}x_{ij_1}^\trp \beta_i\beta_i^\trp x_{ij_2}x_{ij_2}^\trp,\notag \\
F & = \sum_{i=1}^n \frac{w_i}{m_i(m_i-1)}\sum_{j_1\ne j_2}x_{ij_1}x_{ij_1}^\trp \beta_i z_{ij_2}x_{ij_2}^\trp,\notag \\
G & = \sum_{i=1}^n \frac{w_i}{m_i(m_i-1)}\sum_{j_1\ne j_2}x_{ij_1}z_{ij_1}z_{ij_2}x_{ij_2}^\trp.\notag
\end{align}
For every $i = 1,\ldots,n$ and $j = 1,\ldots,m_i$, define $b_{ij} = z_{ij}x_{ij}$ and $h_{ij} = x_{ij}x_{ij}^\trp \beta_i$.
Now we can rewrite $E, F, G$ as $E = E_1 - E_2, F = F_1 - F_2, G = G_1 - G_2$, where
\begin{align*}
E_1 & = \sum_{i=1}^n \frac{w_i}{m_i(m_i-1)}\left(\sum_{j=1}^{m_i}h_{ij}\right)\left(\sum_{j=1}^{m_i}h_{ij}\right)^\trp, &
E_2 & = \sum_{i=1}^n \frac{w_i}{m_i(m_i-1)}\sum_{j=1}^{m_i}h_{ij}h_{ij}^\trp,\\
F_1 & = \sum_{i=1}^n \frac{w_i}{m_i(m_i-1)} \left(\sum_{j=1}^{m_i} h_{ij}\right)\left(\sum_{j=1}^{m_i}b_{ij}\right)^\trp, &
F_2 & = \sum_{i=1}^n \frac{w_i}{m_i(m_i-1)}\sum_{j=1}^{m_i}h_{ij}b_{ij}^\trp,\\
G_1 & = \sum_{i=1}^n \frac{w_i}{m_i(m_i-1)} \left(\sum_{j=1}^{m_i} b_{ij}\right)\left(\sum_{j=1}^{m_i}b_{ij}
\right)^\trp, &
G_2 & = \sum_{i=1}^n \frac{w_i}{m_i(m_i-1)}\sum_{j=1}^{m_i}b_{ij}b_{ij}^\trp.
\end{align*}
Define
\[
r_i = \frac{(d + \log(m_in/\delta))\log(m_in/\delta)}{m_i^2}, \quad
s_i = \sqrt{\frac{d + \log(m_in/\delta)}{m_i^4}} + r_i.
\]
The following lemma controls the deviation of $E_1,E_2,F_1,F_2,G_1,G_2$ from their expectations $\bar E_1,\allowbreak\bar E_2,\bar F_1,\bar F_2,\bar G_1,\bar G_2$ in spectral norm:
\begin{lemma}
\label{lm:linear-model-spectral-terms}
There exists an absolute constant $C > 0$ such that for any $\delta\in(0,1/2)$,
each of the following inequality holds with probability at least $1-\delta$:
\begin{align}
\|E_1 - \bar E_1\| & \le C\sqrt{\sum_{i=1}^n \frac{w_i^2\|\beta_i\|_2^4d\log(d/\delta)}{m_i}} + C\max_iw_i\|\beta_i\|_2^2q_i\log(d/\delta),\label{eq:E1}\\
\|E_2 - \bar E_2\| & \le C\sqrt{\sum_{i=1}^n \frac{w_i^2\|\beta_i\|_2^4d\log(d/\delta)}{m_i^3}} + C\max_i w_i\|\beta_i\|_2^2s_i\log(d/\delta),\label{eq:E2}\\
\|F_1 - \bar F_1\| & \le C\sqrt{\sum_{i=1}^n \frac{w_i^2\eta_i^2\|\beta_i\|_2^2d\log(d/\delta)}{m_i}} + C\max_i w_i\eta_i\|\beta_i\|_2q_i\log(d/\delta),\label{eq:F1}\\
\|F_2 - \bar F_2\| & \le C\sqrt{\sum_{i=1}^n \frac{w_i^2\eta_i^2\|\beta_i\|_2^2d\log(d/\delta)}{m_i^3}} + C\max_i w_i\eta_i\|\beta_i\|_2s_i\log(d/\delta),\label{eq:F2}\\
\|G_1 - \bar G_1\| & \le C\sqrt{\sum_{i=1}^n \frac{w_i^2\eta_i^4d\log(d/\delta)}{m_i^2}} + C\max_i w_i\eta_i^2p_i\log(d/\delta),\label{eq:G1}\\
\|G_2 - \bar G_2\| & \le C\sqrt{\sum_{i=1}^n \frac{w_i^2\eta_i^4d\log(d/\delta)}{m_i^3}} + C\max_i w_i\eta_i^2r_i\log(d/\delta).\label{eq:G2}
\end{align}
\end{lemma}
Before we prove \Cref{lm:linear-model-spectral-terms}, we first use it to prove \Cref{thm:linear-model}:
\begin{proof}[Proof of \Cref{thm:linear-model}]
By \eqref{eq:linear-model-decomposition},
\[
\|A - \bar A\| \le \|E_1 - \bar E_1\| + \|E_2 - \bar E_2\| + \|F_1 - \bar F_1\| + \|F_2 - \bar F_2\| + \|G_1 - \bar G_1\| + \|G_2 - \bar G_2\|.
\]
Setting $\delta$ in \Cref{lm:linear-model-spectral-terms} to be $\delta/6$ and using the union bound, for an absolute constant $C > 0$, with probability at least $1-\delta$,
\begin{align*}
\|A - \bar A\| \le {} & C\sqrt{\sum_{i=1}^nw_i^2d\log(d/\delta)\left(\frac{\|\beta_i\|_2^4 + \eta_i^2\|\beta_i\|_2^2}{m_i} + \frac{\eta_i^4}{m_i^2}\right)}\\
& + C\max_i w_i\log(d/\delta)(\|\beta_i\|_2^2q_i + \eta_i\|\beta_i\|_2 q_i + \eta_i^2p_i),
\end{align*}
where we used the fact that the right-hand-sides of \eqref{eq:E2}, \eqref{eq:F2} and \eqref{eq:G2} are upper bounded by a constant times the right-hand-sides of \eqref{eq:E1}, \eqref{eq:F1}, \eqref{eq:G1}, respectively. Since $\bar A = \sum_{i=1}^n w_i\beta_i\beta_i^\trp$, the theorem is proved by 
\[
\sin\theta \le \frac{2\|A - \bar A\|}{\sigma_k^2}
\]
due to \Cref{thm:DK}.
\end{proof}

In the lemmas below, we prove helper inequalities that we use to prove \Cref{lm:linear-model-spectral-terms}. In these lemmas, we focus on a single user $i$ and thus omit the subscript $i$.
\begin{lemma}
Let $x_1,\ldots,x_m\in\bR^d$ are independent random vectors. Assume every $x_i$ has zero mean and is $O(1)$-sub-Gaussian. Conditioned on $x_1,\ldots,x_m$, let $z_1,\ldots,z_m\in\bR$ be $\eta$-sub-Gaussian independent random variables with zero mean. Define $b = \sum_{j=1}^m z_jx_j$. For any $\delta\in (0,1/2)$, with probability at least $1-\delta$,
\begin{equation}
\label{eq:b-norm}
\|b\|_2 \le \eta\sqrt{(d + \log(1/\delta))(m + \log(1/\delta))}.
\end{equation}
Moreover,
\begin{align}
\|\bE[bb^\trp]\| & \le O(\eta^2m), \label{eq:b-expectation}\\
\|\bE[(bb^\trp)^2]\| & \le O(\eta^4m^2d). \label{eq:b-variance}
\end{align}
\end{lemma}
\begin{proof}
Define $X = [x_1\ \cdots\ x_m] \in\bR^{d\times m}$ and $z = [z_1\ \cdots\ z_m]^\trp\in\bR^m$. Now $b = Xz$. By \Cref{lm:matrix-norm}, with probability at least $1-\delta/2$,
\begin{equation}
\label{eq:b-norm-1}
\|X\| \le O\left(\sqrt{d + m + \log(1/\delta)}\right).
\end{equation}
Let $r$ denote the rank of $X$. 
By \Cref{lm:high-prob-noise-fixed},
with probability at least $1-\delta/2$,

\begin{equation}
\label{eq:b-norm-2}
\|b\|_2 = \|Xz\|_2 \le O\left(\eta\|X\|\sqrt{r + \log(1/\delta)}\right) \le O\left(\eta\|X\|\sqrt{\min\{d,m\} + \log(1/\delta)}\right).
\end{equation}
Combining \eqref{eq:b-norm-1} and \eqref{eq:b-norm-2} using the union bound, with probability at least $1-\delta$,
\[
\|b\|_2 \le O\left(\eta\sqrt{(d + m + \log(1/\delta))(\min\{d,m\} + \log(1/\delta))}\right).
\]
This proves \eqref{eq:b-norm}.

For every fixed unit vector $u\in\bR^d$,
\[
u^\trp \bE[bb^\trp]u = \sum_{j=1}^m u^\trp x_jz_j^2x_j^\trp u = \sum_{j=1}^m z_j^2(x_j^\trp u)^2 \le \sum_{j=1}^m O(\eta^2) = O(\eta^2m).
\]
This proves \eqref{eq:b-expectation}.
\begin{align*}
u^\trp \bE[(bb^\trp)^2]u = \sum_{j_1,j_2,j_3,j_4}\bE[u^\trp z_{j_1}x_{j_1}x_{j_2}^\trp z_{j_2}z_{j_3}x_{j_3}x_{j_4}^\trp z_{j_4}u]
\end{align*}
Since $z_1,\ldots,z_m$ are independent and have zero mean when conditioned on $x_1,\ldots,x_m$, the term $\bE[u^\trp z_{j_1}x_{j_1}x_{j_2}^\trp z_{j_2}z_{j_3}x_{j_3}x_{j_4}^\trp z_{j_4}u]$ is zero unless $(j_1,j_2,j_3,j_4)$ belongs to
\[
S := \{(j_1,j_2,j_3,j_4)\in \{1,\ldots,m\}^4: (j_1 = j_2 \wedge j_3 = j_4) \vee (j_1 = j_3 \wedge j_2 = j_4)  \vee (j_1 = j_4 \wedge j_2 = j_3) \}.
\]
Letting $e_1,\ldots,e_d$ be an orthonormal bases of $\bR^d$, we have
\begin{align*}
u^\trp \bE[(bb^\trp)^2]u & = \sum_{(j_1,j_2,j_3,j_4)\in S}\sum_{\ell = 1}^d\bE[z_{j_1}z_{j_2}z_{j_3}z_{j_4}(u^\trp x_{j_1})(x_{j_2}^\trp e_\ell)(e_\ell^\trp x_{j_3})(x_{j_4}^\trp u)]\\
& \le \sum_{(j_1,j_2,j_3,j_4)\in S}\sum_{\ell = 1}^d\frac {\eta^4}8\bE[(z_{j_1}/\eta)^8 + (z_{j_2}/\eta)^8 + (z_{j_3}/\eta)^8 + (z_{j_4}/\eta)^8\\
& \quad + (u^\trp x_{j_1})^8 + (x_{j_2}^\trp e_\ell)^8 + (e_\ell^\trp x_{j_3})^8 + (x_{j_4}^\trp u)^8]\\
& \le \sum_{(j_1,j_2,j_3,j_4)\in S}\sum_{\ell = 1}^d O(\eta^4)\\
& \le O(\eta^4m^2d).
\end{align*}
This proves \eqref{eq:b-variance}.
\end{proof}
\begin{lemma}
In the same setting as the lemma above, define $h = \sum_{j=1}^m x_jx_j^\trp \beta$ for a vector $\beta\in\bR^d$. For any $\delta\in (0,1/2)$, with probability at least $1-\delta$,
\begin{equation}
\label{eq:h-norm}
\|h\|_2 \le O\left(\|\beta\|_2\sqrt{(d + m + \log(1/\delta))(m + \log(1/\delta)}\right).
\end{equation}
Also, with probability at least $1-\delta$,
\begin{equation}
\label{eq:hh-norm}
\|hh^\trp - \bE[hh^\trp]\|\le O\left(\|\beta\|_2^2\left(
m\sqrt{m(d + \log(1/\delta)} + (d + \log(1/\delta))(m + \log(1/\delta))
\right)\right)
\end{equation}
Moreover, 
\begin{equation}
\label{eq:hh-variance}
\|\bE[(hh^\trp - \bE[hh^\trp])^2]\| \le O(\|\beta\|_2^4m^3d).
\end{equation}
\end{lemma}
\begin{proof}
Define $X = [x_1\ \cdots\ x_m] = \bR^{d\times m}$. We have $h = XX^\trp \beta$. 
Since each of the $m$ entries in $X^\trp \beta$ has sub-Gaussian constant $O(\|\beta\|_2)$,
by \Cref{lm:matrix-norm}, with probability at least $1-\delta/2$, $\|X^\trp \beta\|_2 \le O(\|\beta\|_2\sqrt{m + \log(1/\delta)})$. By \Cref{lm:matrix-norm}, with probability at least $1-\delta/2$, $\|X\| \le O(\sqrt{d + m + \log(1/\delta)})$. This proves \eqref{eq:h-norm} by a union bound.

We show that with probability at least $1-\delta$,
\begin{equation}
\label{eq:h-norm-2}
\|h - \bE[h]\|_2 \le O\left(\|\beta\|_2\sqrt{(d + \log(1/\delta))(m + \log(1/\delta))}\right).
\end{equation}
If $d\ge m$, this follows from \eqref{eq:h-norm} and 
\begin{equation}
\label{eq:h-expectation}
\|\bE[h]\|_2 = O(m\|\beta\|_2).
\end{equation}
If $d \le m$, with probability at least $1-\delta$,
\begin{align*}
\|h - \bE[h]\|_2 & = \|(XX^\trp - \bE[XX^\trp])\beta\|_2\\
& \le \|XX^\trp - \bE[XX^\trp]\|\cdot \|\beta\|_2\\
& \le O\left(\|\beta\|_2 \sqrt{(d + \log(1/\delta))(m + d + \log(1/\delta))}\right)\tag{by \Cref{lm:matrix-concentration}}\\
& \le O\left(\|\beta\|_2 \sqrt{(d + \log(1/\delta))(m + \log(1/\delta))}\right).
\end{align*}
Therefore, inequality \eqref{eq:h-norm-2} holds with probability at least $1-\delta$. Now
we show that $\|\bE[h]\bE[h]^\trp - \bE[hh^\trp]\| \le O(m\|\beta\|_2^2)$. Indeed,
\[
\bE[h]\bE[h]^\trp - \bE[hh^\trp] = \sum_{j_1,j_2}\left(\bE[x_{j_1}x_{j_1}^\trp \beta]\bE[x_{j_2}x_{j_2}^\trp \beta]^\trp - \bE[(x_{j_1}x_{j_1}^\trp \beta)(x_{j_2}x_{j_2}^\trp \beta)^\trp]\right).
\]
Since $x_{j_1}$ and $x_{j_2}$ are independent when $j_1\ne j_2$,
\begin{equation}
\label{eq:h-norm-3}
\bE[h]\bE[h]^\trp - \bE[hh^\trp] = \sum_{j=1}^m\left(\bE[x_{j}x_{j}^\trp \beta]\bE[x_{j}x_{j}^\trp \beta]^\trp - \bE[(x_{j}x_{j}^\trp \beta)(x_{j}x_{j}^\trp \beta)^\trp]\right).
\end{equation}
It is clear that $\bE[x_jx_j^\trp\beta] = \beta$, and for every unit vector $u\in\bR^d$,
\begin{align*}
& u^\trp \bE[(x_{j}x_{j}^\trp \beta)(x_{j}x_{j}^\trp \beta)^\trp] u\\
= {} & \bE[(u^\trp x_j)(x_j^\trp \beta)(\beta^\trp x_j)(x_j^\trp u)]\\
\le {} & \frac{\|\beta\|_2^2}{4}\bE[(u^\trp x_j)^4 + (x_j^\trp (\beta/\|\beta\|_2))^4 + ((\beta/\|\beta\|_2)^\trp x_j)^4 + (x_j^\trp u)^4]\\
\le {} & O(\|\beta\|_2^2).
\end{align*}
This implies that $\|\bE[(x_{j}x_{j}^\trp \beta)(x_{j}x_{j}^\trp \beta)^\trp]\| \le O(\|\beta\|_2^2)$.
Now it is clear from \eqref{eq:h-norm-3} that $\|\bE[h]\bE[h]^\trp - \bE[hh^\trp]\| \le O(m\|\beta\|_2^2)$, which implies
\begin{align*}
\|hh^\trp - \bE[hh^\trp]\| & \le O(m\|\beta\|_2^2) + \|hh^\trp - \bE[h]\bE[h]^\trp\|\\
& \le O(m\|\beta\|_2^2) + \|(h - \bE[h])h^\trp\| + \|\bE[h](h - \bE[h])^\trp\|\\
& \le O(m\|\beta\|_2^2) + \|h - \bE[h]\|_2 (\|h\|_2 + \|\bE[h]\|_2).
\end{align*}
Plugging \eqref{eq:h-norm-2}, \eqref{eq:h-norm}, and \eqref{eq:h-expectation} into the inequality above, we get \eqref{eq:hh-norm}.

Finally, 
\begin{equation}
\label{eq:hh-variance-1}
\bE[(hh^\trp - \bE[hh^\trp])^2] = 
\bE[(hh^\trp)^2] - \bE[hh^\trp]^2 = \sum_{j_1,j_2,j_3,j_4}H_{j_1j_2j_3j_4},
\end{equation}
where
\[
H_{j_1j_2j_3j_4}
= \bE[x_{j_1}x_{j_1}^\trp \beta\beta^\trp x_{j_2}x_{j_2}^\trp x_{j_3}x_{j_3}^\trp \beta\beta^\trp x_{j_4}x_{j_4}^\trp] -\bE[x_{j_1}x_{j_1}^\trp \beta\beta^\trp x_{j_2}x_{j_2}^\trp]\bE[x_{j_3}x_{j_3}^\trp \beta\beta^\trp x_{j_4}x_{j_4}^\trp].
\]
When $j_1,j_2,j_3,j_4$ are distinct, $H_{j_1j_2j_3j_4}$ is zero. When $j_1,j_2,j_3,j_4$ are not distinct, letting $e_1,\ldots,e_d$ be an orthonormal basis for $\bR^d$, for every fixed unit vector $u\in\bR^d$, we have 
\begin{align*}
& u^\trp \bE[x_{j_1}x_{j_1}^\trp \beta\beta^\trp x_{j_2}x_{j_2}^\trp x_{j_3}x_{j_3}^\trp \beta\beta^\trp x_{j_4}x_{j_4}^\trp]u\\
= {} & \sum_{\ell = 1}^d\bE[(u^\trp x_{j_1})(x_{j_1}^\trp\beta)(\beta^\trp x_{j_2})(x_{j_2}^\trp e_\ell)(e_\ell^\trp x_{j_3})(x_{j_3}^\trp \beta)(\beta^\trp x_{j_4})(x_{j_4}^\trp u)]\\
\le {} & \sum_{\ell = 1}^d \frac{\|\beta\|_2^4}{8}\bE[(u^\trp x_{j_1})^8 + (x_{j_1}^\trp(\beta/\|\beta\|_2))^8 + ((\beta/\|\beta\|_2)^\trp x_{j_2})^8 + (x_{j_2}^\trp e_\ell)^8 \\
& + (e_\ell^\trp x_{j_3})^8 + (x_{j_3}^\trp (\beta/\|\beta\|_2))^8 + ((\beta/\|\beta\|_2)^\trp x_{j_4})^8 + (x_{j_4}^\trp u)^8]\\
= {} & O(\|\beta\|_2^4d).
\end{align*}
This implies $\|\bE[x_{j_1}x_{j_1}^\trp \beta\beta^\trp x_{j_2}x_{j_2}^\trp x_{j_3}x_{j_3}^\trp \beta\beta^\trp x_{j_4}x_{j_4}^\trp]\| = O(\|\beta\|_2^4d)$.
Similarly,
$\|\bE[x_{j_1}x_{j_1}^\trp\beta\beta^\trp x_{j_2}x_{j_2}^\trp]\| = O(\|\beta\|_2^2)$ because
\begin{align*}
u^\trp \bE[x_{j_1}x_{j_1}^\trp\beta\beta^\trp x_{j_2}x_{j_2}^\trp]u & = \bE[(u^\trp x_{j_1})(x_{j_1}^\trp \beta)(\beta^\trp x_{j_2})(x_{j_2}^\trp u)]\\
& \le \frac{\|\beta\|_2^2}{4}\bE[(u^\trp x_{j_1})^4 + (x_{j_1}^\trp (\beta/\|\beta\|_2))^4 + ((\beta/\|\beta\|_2)^\trp x_{j_2})^4 + (x_{j_2}^\trp u)^4]\\
& \le O(\|\beta\|_2^2).
\end{align*}
Therefore, $\|H_{j_1j_2j_3j_4}\|\le O(\|\beta\|_2^4 d)$ when $j_1,j_2,j_3,j_4$ are not distinct. Now \eqref{eq:hh-variance} follows from \eqref{eq:hh-variance-1}.
\end{proof}
\begin{lemma}
In the lemma above, 
\begin{equation}
\label{eq:expectation-cross}
\|\bE[hb^\trp b h^\trp]\|\le O(\eta^2\|\beta\|_2^2m^3d), \textnormal{ and }\|\bE[bh^\trp hb^\trp]\| \le O(\eta^2\|\beta\|_2^2m^3d).
\end{equation}
\end{lemma}
\begin{proof}
For every unit vector $u$,
\begin{equation}
\label{eq:proof-expectation-cross-1}
u^\trp \bE[hb^\trp bh^\trp]u = \sum_{j_1,j_2,j_3,j_4}\bE[u^\trp x_{j_1}x_{j_1}^\trp\beta z_{j_2}x_{j_2}^\trp x_{j_3}z_{j_3}\beta x_{j_4}x_{j_4}^\trp u]
\end{equation}
If $j_2\ne j_3$, we have $\bE[z_{j_2}z_{j_3}|x_{j_1},x_{j_2},x_{j_3},x_{j_4}] = 0$ and thus
$\bE[u^\trp x_{j_1}x_{j_1}^\trp\beta z_{j_2}x_{j_2}^\trp x_{j_3}z_{j_3}\beta x_{j_4}x_{j_4}^\trp u] = 0$.
If $j_2 = j_3$, letting $e_1,\ldots,e_d$ be an orthonormal basis for $\bR^d$,
\begin{align*}
& \bE[u^\trp x_{j_1}x_{j_1}^\trp\beta z_{j_2}x_{j_2}^\trp x_{j_3}z_{j_3}\beta x_{j_4}x_{j_4}^\trp u]\\
= {} & \sum_{\ell = 1}^d \bE[z_{j_2}^2 (u^\trp x_{j_1})(x_{j_1}^\trp\beta) (x_{j_2}^\trp e_{\ell})(e_{\ell}^\trp x_{j_3})(\beta^\trp x_{j_4})(x_{j_4}^\trp u)]\\
\le {} & \frac{\|\beta\|_2^2}{6}\sum_{\ell = 1}^d\bE[z_{j_2}^2((u^\trp x_{j_1})^6 + (x_{j_1}^\trp(\beta/\|\beta\|_2))^6 + (x_{j_2}^\trp e_{\ell})^6 \\
& + (e_{\ell}^\trp x_{j_3})^6 + ((\beta/\|\beta\|_2)^\trp x_{j_4})^6 + (x_{j_4}^\trp u)^6)]\\
\le {} & O(d\eta^2\|\beta\|_2^2).
\end{align*}
Plugging this into \eqref{eq:proof-expectation-cross-1} and noting that the size of $\{(j_1,j_2,j_3,j_4)\in \{1,\ldots,m\}^4: j_2 = j_3\}$ is $O(m^3)$, we get $\|\bE[hb^\trp b h^\trp]\|\le O(\eta^2\|\beta\|_2^2m^3d)$.

Similarly,
\begin{equation}
\label{eq:proof-expectation-cross-2}
u^\trp \bE[bh^\trp hb^\trp]u = \sum_{j_1,j_2,j_3,j_4}\bE[u^\trp z_{j_1}x_{j_1}\beta^\trp x_{j_2}x_{j_2}^\trp x_{j_3}x_{j_3}^\trp \beta z_{j_4}x_{j_4}^\trp u].
\end{equation}
If $j_1 \ne j_4$, then $\bE[u^\trp z_{j_1}x_{j_1}\beta^\trp x_{j_2}x_{j_2}^\trp x_{j_3}x_{j_3}^\trp \beta z_{j_4}x_{j_4}^\trp u] = 0$. When $j_1 = j_4$,
\begin{align*}
& \bE[u^\trp z_{j_1}x_{j_1}\beta^\trp x_{j_2}x_{j_2}^\trp x_{j_3}x_{j_3}^\trp \beta z_{j_4}x_{j_4}^\trp u]\\
= {} &  \sum_{\ell = 1}^d\bE[z_{j_1}^2(u^\trp x_{j_1})(\beta^\trp x_{j_2})(x_{j_2}^\trp e_{\ell})(e_{\ell}^\trp x_{j_3})(x_{j_3}^\trp \beta) (x_{j_4}^\trp u)]\\
\le {} & \frac{\|\beta\|_2^2}6 \sum_{\ell = 1}^d\bE[z_{j_1}^2((u^\trp x_{j_1})^6 + ((\beta/\|\beta\|_2)^\trp x_{j_2})^6 + (x_{j_2}^\trp e_{\ell})^6\\
& + (e_{\ell}^\trp x_{j_3})^6 + (x_{j_3}^\trp (\beta/\|\beta\|_2))^6 + (x_{j_4}^\trp u)^6)]\\
= {} & O(d\eta^2\|\beta\|_2^2).
\end{align*}
Plugging this into \eqref{eq:proof-expectation-cross-2} proves $\|\bE[bh^\trp hb^\trp]\| \le O(\eta^2\|\beta\|_2^2m^3d)$.
\end{proof}
We can now finish the proof of \Cref{lm:linear-model-spectral-terms}.
\begin{proof}[Proof of \Cref{lm:linear-model-spectral-terms}]
To bound $\|E_1 - \bar E_1\|$, we set $Z_i' = \frac{w_i}{m_i(m_i-1)}\left(\sum_{j=1}^{m_i}h_{ij}\right)\left(\sum_{j=1}^{m_i}h_{ij}\right)^\trp$ and $Z_i = Z_i' - \bE[Z_i']$. By \eqref{eq:hh-norm}, with probability at least $1 - \delta/(2n)$,
\[
\|Z_i\|\le O\left(w_i\|\beta_i\|_2^2q_i\right).
\]
Also, by \eqref{eq:hh-variance},
\[
\|\bE[Z_iZ_i^\trp]\| = \|\bE[Z_i^\trp Z_i]\| = O\left(\frac{w_i^2\|\beta_i\|_2^4d}{m_i}\right).
\]
\Cref{lm:extended-matrix-bernstein} proves that \eqref{eq:E1} holds with probability at least $1-\delta$.

To bound $\|E_2 - \bar E_2\|$, we set $Z_{ij}' = \frac{w_i}{m_i(m_i - 1)}h_{ij}h_{ij}^\trp$ and $Z_{ij} = Z_{ij}' - \bE[Z_{ij}']$. By \eqref{eq:hh-norm}, with probability at least $1 - \delta/(2nm_i)$,
\[
\|Z_{ij}\| \le O(w_i\|\beta_i\|_2^2 s_i).
\]
Also, by \eqref{eq:hh-variance},
\[
\|\bE[Z_{ij}Z_{ij}^\trp]\| = \|\bE[Z_{ij}^\trp Z_{ij}]\| = O\left(\frac{w_i^2\|\beta_i\|_2^4 d}{m_i^4}\right).
\]
\Cref{lm:extended-matrix-bernstein} proves that \eqref{eq:E2} holds with probability at least $1-\delta$.

To bound $\|F_1 - \bar F_1\|$, we set $Z_i = \frac{w_i}{m_i(m_i-1)}\left(\sum_{j=1}^{m_i}h_{ij}\right)\left(\sum_{j=1}^{m_i}b_{ij}\right)^\trp$. Note that $\bE[Z_i] = 0$. By \eqref{eq:b-norm} and \eqref{eq:h-norm}, with probability at least $1 - \delta/(2n)$,
\[
\|Z_i\|\le O\left(w_i\eta_i\|\beta_i\|_2q_i\right).
\]
Also, by \eqref{eq:expectation-cross},
\[
\max\left\{\|\bE[Z_iZ_i^\trp]\|, \|\bE[Z_i^\trp Z_i]\|\right\} = O\left(\frac{w_i^2\eta_i^2\|\beta_i\|_2^2d}{m_i}\right).
\]
\Cref{lm:extended-matrix-bernstein} proves that \eqref{eq:F1} holds with probability at least $1-\delta$.

To bound $\|F_2 - \bar F_2\|$, we set $Z_{ij} = \frac{w_i}{m_i(m_i - 1)}h_{ij}b_{ij}^\trp$. Note that $\bE[Z_{ij}] = 0$. By \eqref{eq:b-norm} and \eqref{eq:h-norm}, with probability at least $1 - \delta/(2nm_i)$,
\[
\|Z_{ij}\| \le O(w_i\eta_i\|\beta_i\|_2 s_i).
\]
Also, by \eqref{eq:expectation-cross},
\[
\max\left\{\|\bE[Z_{ij}Z_{ij}^\trp]\|, \|\bE[Z_{ij}^\trp Z_{ij}]\|\right\} = O\left(\frac{w_i^2\eta_i^2\|\beta_i\|_2^2 d}{m_i^4}\right).
\]
\Cref{lm:extended-matrix-bernstein} proves that \eqref{eq:F2} holds with probability at least $1-\delta$.

To bound $\|G_1 - \bar G_1\|$, we set $Z_i' = \frac{w_i}{m_i(m_i-1)}\left(\sum_{j=1}^{m_i}b_{ij}\right)\left(\sum_{j=1}^{m_i}b_{ij}\right)^\trp$ and $Z_i = Z_i' - \bE[Z_i']$. By \eqref{eq:b-norm} and \eqref{eq:b-expectation}, with probability at least $1 - \delta/(2n)$,
\[
\|Z_i\|\le O\left(w_i\eta_i^2p_i\right).
\]
Also, by \eqref{eq:b-variance},
\[
\|\bE[Z_iZ_i^\trp]\| = \|\bE[Z_i^\trp Z_i]\| = \|\bE[Z_i^2]\| \le \|\bE[(Z_i')^2]\| = O\left(\frac{w_i^2\eta^4d}{m_i^2}\right).
\]
\Cref{lm:extended-matrix-bernstein} proves that \eqref{eq:G1} holds with probability at least $1-\delta$.

To bound $\|G_2 - \bar G_2\|$, we set $Z_{ij}' = \frac{w_i}{m_i(m_i - 1)}b_{ij}b_{ij}^\trp$ and $Z_{ij} = Z_{ij}' - \bE[Z_{ij}']$. By \eqref{eq:b-norm} and \eqref{eq:b-expectation}, with probability at least $1 - \delta/(2nm_i)$,
\[
\|Z_{ij}\| \le O(w_i\eta_i^2 r_i).
\]
Also, by \eqref{eq:b-variance},
\[
\|\bE[Z_{ij}Z_{ij}^\trp]\| = \|\bE[Z_{ij}^\trp Z_{ij}]\| = \|\bE[Z_{ij}^2]\| \le \|\bE[(Z_{ij}')^2]\| = O\left(\frac{w_i^2\eta_i^4 d}{m_i^4}\right).
\]
\Cref{lm:extended-matrix-bernstein} proves that \eqref{eq:G2} holds with probability at least $1-\delta$.
\end{proof}
\subsection{Proof of Corollary~\ref{cor:linear-model}}
\label{sec:proof-cor-linear-model}
\begin{proof}
We note that
\[
p_i \le O\left(\frac{d\log^2(n/\delta)}{m}\right), \quad q_i \le p_i + O\left(\sqrt{\frac{d\log(n/\delta)}{m}}\right).
\]
Therefore, by \Cref{thm:linear-model},
with probability at least $1-\delta$,
\begin{align*}
\sin\theta \le {} & O\left(\log(d/\delta)\sqrt{\log(n/\delta)}\sqrt{\frac{d(\frac 1n\sum_{i=1}^n \|\beta_i\|_2^4 + \frac 1n\sum_{i=1}^n\eta^2\|\beta_i\|_2^2 + \eta^4/m)}{mn\sigma_k^2}}\right)\\
& + O\left(\log(d/\delta)\log^2(n/\delta)\frac{d(\eta^2 + \max_i \|\beta_i\|_2^2)}{mn\sigma_k^2}\right).
\end{align*}
Since $\sin\theta\le 1$ always holds, the inequality above implies
\begin{align}
\sin\theta
\le {} & O\left(\log^3(nd/\delta)\left(\sqrt{\frac{d(\frac 1n\sum_{i=1}^n\|\beta_i\|_2^4 + \frac 1n\sum_{i=1}^n\eta^2\|\beta_i\|_2^2 + \eta^4/m)}{mn\sigma_k^4}} + \frac{d\max_i\|\beta_i\|_2^2}{mn\sigma_k^2}\right)\right)\notag \\
& + O\left(\min\left\{1,\frac{d\eta^2\log(d/\delta)\log^2(n/\delta)}{mn\sigma_k^2}\right\}\right).\label{eq:cor-linear-model-3}
\end{align}
Using $\frac 1n\sum_{i=1}^n \|\beta_i\|_2^2 = \tr(\frac 1n\sum_{i=1}^n \beta_i\beta_i^\trp) = \tr(\sum_{i=1}^n w_i\beta_i\beta_i^\trp) \ge \sum_{\ell = 1}^k \sigma_\ell^2 \ge k\sigma_k^2$,
we have
\begin{align*}
\min\left\{1,\frac{d\eta^2\log(d/\delta)\log^2(n/\delta)}{mn\sigma_k^2}\right\} & \le \sqrt{\frac{d\eta^2\log(d/\delta)\log^2(n/\delta)}{mn\sigma_k^2}}\\
& \le \sqrt{\frac{d(\frac 1n\sum_{i=1}^n\eta^2\|\beta_i\|_2^2)\log(d/\delta)\log^2(n/\delta)}{mn\sigma_k^4}}.
\end{align*}
Plugging this into \eqref{eq:cor-linear-model-3} proves \eqref{eq:cor-linear-model-1}. When $\|\beta_i\|_2 \le r$ for every $i = 1,\ldots,n$, \eqref{eq:cor-linear-model-1} implies
\[
\sin\theta \le O\left(\log^3(nd/\delta)\left(\sqrt{\frac{d(r^4 + r^2\eta^2 + \eta^4/m)}{mn\sigma_k^4}} + \frac{dr^2}{mn\sigma_k^2}\right)\right).
\]
Since $\sin\theta \le 1$, the inequality above implies
\begin{equation}
\label{eq:cor-linear-model-4}
\sin\theta \le O\left(\log^3(nd/\delta)\left(\sqrt{\frac{d(r^4 + r^2\eta^2 + \eta^4/m)}{mn\sigma_k^4}} + \min\left\{1,\frac{dr^2}{mn\sigma_k^2}\right\}\right)\right).
\end{equation}
Using $r^2 \ge \frac 1n \sum_{i=1}^n \|\beta_i\|_2^2 \ge k\sigma_k^2$,
\[
\min\left\{1,\frac{dr^2}{mn\sigma_k^2}\right\} \le \sqrt{\frac{dr^2}{mn\sigma_k^2}} \le \sqrt{\frac{dr^4}{mn\sigma_k^4}}.
\]
Plugging this into \eqref{eq:cor-linear-model-4} proves \eqref{eq:cor-linear-model-2}.
\end{proof}
\bibliography{ref}
\end{document}